%% file: main.tex
\def\year{2021}\relax
\documentclass[letterpaper]{article} %
\usepackage{aaai21}  %
\usepackage{times}  %
\usepackage{helvet} %
\usepackage{courier}  %
\usepackage[hyphens]{url}  %
\usepackage{graphicx} %
\usepackage{natbib}  %
\usepackage{caption} %
\usepackage[switch]{lineno} %

\usepackage{amsmath}
\usepackage{amssymb}
\usepackage{amsfonts}
\usepackage{amsthm}
\usepackage{latexsym}
\usepackage{cite} 
\usepackage{color,xspace,colortbl}
\usepackage{algorithm}
\usepackage[noend]{algpseudocode}

\usepackage{array}
\usepackage{tikz,tikz-qtree,tikz-qtree-compat,tikz-3dplot}                      
\usepackage{bm}

\input{macros}

\title{On the Tractability of SHAP Explanations}
\author {
    Guy Van den Broeck,\textsuperscript{\rm 1}
    Anton Lykov,\textsuperscript{\rm 1}
    Maximilian Schleich,\textsuperscript{\rm 2}
    Dan Suciu\textsuperscript{\rm 2} \\
}
\affiliations {
    \textsuperscript{\rm 1} University of California, Los Angeles \\
    \textsuperscript{\rm 2} University of Washington, Seattle 
}

\begin{document}

\maketitle

\begin{abstract}

\shap\ explanations are a popular feature-attribution mechanism for explainable AI. 
They use game-theoretic notions to measure the influence of individual features on
the prediction of a machine learning model. Despite a lot of recent interest
from both academia and industry, it is not known whether \shap\ explanations of common machine learning models can be computed efficiently. 
In this paper, we establish the complexity of computing the \shap\ explanation in three important settings.
First, we consider fully-factorized data distributions, and show that the
complexity of computing the \shap\ explanation is the same as the complexity of
computing the expected value of the model. This fully-factorized setting is often used to
simplify the \shap\ computation, yet our results show that the computation can
be intractable for commonly used models such as logistic regression.
Going beyond fully-factorized distributions, we show that computing
\shap\ explanations is already intractable for a very simple setting: computing
\shap\ explanations of trivial classifiers over naive Bayes distributions. Finally, we show that even computing \shap\ over the empirical distribution is \#P-hard.
\end{abstract}

\input{sec1-intro}

\input{sec2-problem}
\input{sec4-wmc}

\input{sec5-beyondindep}

\input{sec6-empirical}

\input{sec7-conclusion}

\subsection*{Acknowledgements}
This work is partially supported by NSF grants IIS-1907997, IIS-1954222 IIS-1943641, IIS-1956441, CCF-1837129, DARPA grant N66001-17-2-4032, a Sloan Fellowship, and gifts by Intel and Facebook research.
Schleich is supported by a RelationalAI fellowship. The authors would like to thank YooJung Choi for valuable discussions on the proof of Theorem~\ref{th:logistic:regression:hard}.

\bibliography{main}

\clearpage
\appendix

\input{secA}

\input{secC}

\input{secD}

\input{secE}

\input{secF}

\input{secG}

\end{document}

%% file: macros.tex
\newcolumntype{L}[1]{>{\raggedright\let\newline\\\arraybackslash\hspace{0pt}}p{#1}}
\newcolumntype{C}[1]{>{\centering\let\newline\\\arraybackslash\hspace{0pt}}p{#1}}
\newcolumntype{R}[1]{>{\raggedleft\let\newline\\\arraybackslash\hspace{0pt}}p{#1}}

\usetikzlibrary{shapes,decorations.markings,arrows.meta,fit,calc,positioning,shapes.gates.logic.US}

\colorlet{LightViolet}{violet!40}
\colorlet{LightRed}{red!40}
\colorlet{LightOrange}{orange!40}
\colorlet{LightGreen}{green!40}
\colorlet{LightBlue}{blue!40}
\colorlet{DarkGreen}{green!50!black}
\colorlet{DarkRed}{red!70!black}
\colorlet{DarkCyan}{red!70!black}
\colorlet{DarkBlue}{blue!80!black}
\definecolor{DarkOrange}{rgb}{1.0, 0.49, 0.0}
\definecolor{Airforceblue}{rgb}{0.36, 0.54, 0.66}

\newcommand{\defeq}{\stackrel{\mathrm{def}}{=}}

\newcommand{\functionname}[1]{\operatorname{#1}} %

\newcommand{\dom}{\functionname{dom}}

\newcommand{\vf}{\functionname{v}}
\newcommand{\cf}{\functionname{c}}

\newcommand{\R}{\mathbb R} %

\newcommand{\pr}{{\textnormal{Pr}}}
\newcommand{\E}{\mathbf{E}}

\newcommand{\be}{\begin{enumerate}}
\newcommand{\ee}{\end{enumerate}}
\newcommand{\bi}{\begin{itemize}}
\newcommand{\ei}{\end{itemize}}
\newcommand{\beq}{\begin{equation}}
\newcommand{\eeq}{\end{equation}}

\newcommand{\bp}{\begin{proof}}
\newcommand{\ep}{\end{proof}}
\newcommand{\bcor}{\begin{cor}}
\newcommand{\ecor}{\end{cor}}
\newcommand{\bthm}{\begin{thm}}
\newcommand{\ethm}{\end{thm}}
\newcommand{\blmm}{\begin{lmm}}
\newcommand{\elmm}{\end{lmm}}
\newcommand{\bdefn}{\begin{defn}}
\newcommand{\edefn}{\end{defn}}
\newcommand{\bprop}{\begin{prop}}
\newcommand{\eprop}{\end{prop}}
\newcommand{\bconj}{\begin{conj}}
\newcommand{\econj}{\end{conj}}
\newcommand{\bopm}{\begin{opm}}
\newcommand{\eopm}{\end{opm}}
\newcommand{\bprm}{\begin{prm}}
\newcommand{\eprm}{\end{prm}}
\newcommand{\brmk}{\begin{rmk}}
\newcommand{\ermk}{\end{rmk}}
\newcommand{\bclm}{\begin{claim}}
\newcommand{\eclm}{\end{claim}}
\newcommand{\bex}{\begin{ex}}
\newcommand{\eex}{\end{ex}}

\theoremstyle{plain}            %
\newtheorem{thm}{Theorem}%
\newtheorem{lmm}[thm]{Lemma}
\newtheorem{prop}[thm]{Proposition}
\newtheorem{cor}[thm]{Corollary}

\theoremstyle{definition}              %

\newtheorem{opm}[thm]{Open Problem}
\newtheorem{prm}[thm]{Problem}
\newtheorem{conj}[thm]{Conjecture}
\newtheorem{ex}[thm]{Example}

\newtheorem{defn}{Definition}

\newtheorem{rmk}[thm]{Remark}

\newcounter{claimCounter}
\newtheorem{claim}[claimCounter]{Claim}

\newtheorem*{example*}{Example}

\newtheorem*{lmm*}{Lemma}

\definecolor{Red}{RGB}{255,204,204}
\definecolor{Green}{RGB}{204,255,204}
\definecolor{Blue}{RGB}{204,204,255}

\newcommand{\set}[1]{\{#1\}}                    %
\newcommand{\setof}[2]{\{{#1}\mid{#2}\}}        %

\newcommand{\shap}{\textsc{Shap}}

\newcommand{\fshap}{\text{\tt F-SHAP}}
\newcommand{\dshap}{\text{\tt D-SHAP}}

\newcommand{\nbn}{\text{\tt NBN}}

\newcommand{\ppcnf}{\text{\tt PP2CNF}}
\newcommand{\logit}{\text{\tt LOGIT}}
\newcommand{\numpar}{\text{\tt NUMPAR}}

\newcommand{\pbsp}{\text{\tt PR}}
\newcommand{\indep}{\text{\tt IND}}
\newcommand{\indeptwo}{\text{\tt INDB}}

\newcommand{\expprob}{\text{\tt E}}
\newcommand{\expquasiprob}{\text{\tt EQS}}

\newcommand{\ttx}{\text{\tt X}}

\newcommand{\ttf}{\text{\tt F}}

\newcommand{\tta}{\text{\tt A}}
\newcommand{\ttb}{\text{\tt B}}

\newcommand{\e}{\mathbf{e}}

%% file: sec1-intro.tex
\section{Introduction}
\label{sec:intro}

\noindent
Machine learning is increasingly applied in high stakes decision making. As a
consequence, there is growing demand for the ability to explain the prediction
of machine learning models. One popular explanation technique is to compute
feature-attribution scores, in particular using the Shapley values from cooperative game theory
\citep{roth1988shapley} as a principled aggregation measure to determine
the influence of individual features on the prediction of the collective model. Shapley
value based explanations have several desirable properties \citep{DBLP:conf/sp/DattaSZ16}, which is why they have attracted a lot of
interest in academia as well as industry in
recent years (see e.g., \citet{xai-industry}).

\citet{strumbelj2014} show that Shapley values can be used to explain arbitrary
machine learning models. \citet{DBLP:conf/sp/DattaSZ16} use Shapley-value-based
explanations as part of a broader framework for algorithmic
transparency. \citet{DBLP:conf/nips/LundbergL17} use Shapley values in a
framework that unifies various explanation techniques, and they coined the term
\shap\ explanation.  They show that the \shap\ explanation is effective in explaining
predictions in the medical domain; see~\citet{lundberg-2020}.  More recently
there has been a lot of work on the tradeoffs of variants of the original
\shap\ explanations, e.g., \citet{sundararajan2019}, \citet{kumar2020problems},
\citet{janzing20a}, \citet{merrick2020}, and \citet{aas2019}.

Despite all of this interest, there is considerable confusion about the
tractability of computing \shap\ explanations. The \shap\ explanations determine the
influence of a given feature by systematically computing the expected value of
the model given a subsets of the features. As a consequence, the complexity of computing
\shap\ explanations depends on the predictive model as well as assumptions on the underlying data distribution.  \citet{lundberg-2020} describe a polynomial-time algorithm
for computing the \shap\ explanation over decision trees, but online discussions
have pointed out that this algorithm is not correct as stated.  We present a
concrete example of this shortcoming in the supplementary material,
in Appendix~\ref{sec:bug}. In contrast, for fully-factorized
distributions, \citet{BLSSV:deem:20} prove that there are models for which
computing the \shap\ explanation is \#P-hard. A contemporaneous paper by
\citet{arenas2020tractability} shows that computing the \shap\ explanation for tractable logical circuits over uniform and fully factorized binary data distributions is tractable. In general, the complexity of the \shap\ explanation is open.

In this paper we consider the original formulation of the \shap\ explanation by
\citet{DBLP:conf/nips/LundbergL17} and analyze its computational complexity under
the following data distributions and model classes:
\begin{enumerate}
\item First, we consider fully-factorized distributions, which are the simplest possible data distribution. Fully-factorized distributions capture the
  assumption that the model's features are independent, which is a commonly used
  assumption to simplify the computation of the \shap\ explanations, see for example
  \citet{DBLP:conf/nips/LundbergL17}.

  For fully-factorized distributions and any prediction model, we show that the
  complexity of computing the \shap\ explanation is the same as the complexity of
  computing the expected value of the model.

  It follows that there are classes of models for which the computation is
  tractable (e.g., linear regression, decision trees, tractable circuits) while for other models,
  including commonly used ones such as logistic regression and neural nets with sigmoid activation functions, it is \#P-hard.

\item Going beyond fully-factorized distributions, we show that computing \shap\ explanation becomes intractable already for the simplest
  probabilistic model that does not assume feature independence: naive Bayes.
  As a consequence, the complexity of computing \shap\ explanations on such data distributions is also intractable for many classes of models, including linear and
  logistic regression.

\item Finally we consider the empirical distribution, and prove that computing \shap\ explanations is \#P-hard for this class of distributions. This result implies that the algorithm by \citet{lundberg-2020} cannot be fixed to compute the exact \shap\ explanations over decision trees in polynomial
  time.
\end{enumerate}

%% file: sec2-problem.tex
\section{Background and Problem Statement}
\label{sec:shap}

Suppose our data is described by $n$ indexed features $\mathbf{X} = \set{X_1, \ldots, X_n}$. Each feature variable $X$ takes a value from a finite domain $\dom(X)$.
A data instance $\mathbf{x} = (x_1, \dots, x_n)$ consists of values
$x \in \dom(X)$ for every feature $X$. This instance space is denoted
$\mathbf{x} \in \mathcal{X} = \dom(X_1) \times \dots \times
\dom(X_n)$. We are also given a learned function
$F: \mathcal{X} \to \mathbb{R}$ that computes a prediction
$F(\mathbf{x})$ on each instance $\mathbf{x}$. Throughout this paper
we assume that the prediction $F(\mathbf{x})$ can be computed in
polynomial time in $n$.

For a particular instance of prediction $F(\mathbf{x})$, the goal of
\textit{local explanations} is to clarify why the function $F$ gave
its prediction on instance $\mathbf{x}$, usually by attributing credit
to the features.  We will focus on local explanation that are inspired
by game-theoretic Shapley values~\citep{DBLP:conf/sp/DattaSZ16,
  DBLP:conf/nips/LundbergL17}.  Specifically, we will work with the
\shap\ explanations as defined by \citet{DBLP:conf/nips/LundbergL17}.

\subsection{\shap\ Explanations}

To produce \shap\ explanations, one needs an additional ingredient: a probability distribution $\Pr(\mathbf{X})$ over the features, which we call the data distribution.
We will use this distribution to reason about partial instances. Concretely, for a set of indices $S \subseteq [n] = \set{1,\dots,n}$, we let $\mathbf{x}_S$ denote the restriction of complete instance $\mathbf{x}$ to those features $\mathbf{X}_S$ with indices in $S$.
Abusing notation, we will also use $\mathbf{x}_S$
to denote the probabilistic event $\mathbf{X}_S = \mathbf{x}_S$.

Under this data distribution, it now becomes possible to ask for the \textit{expected value} of the predictive function~$F$.
Clearly, for a complete data instance $\mathbf{x}$ we have that $\E[F| \mathbf{x}] = F(\mathbf{x})$, as there is no uncertainty about the features.
However, for a partial instance $\mathbf{x}_S$, which does not assign values to the features outside of $\mathbf{X}_S$, we appeal to the data distribution $\Pr$ to compute the expectation of function $F$ as $\E_{\Pr}[F|\mathbf{x}_S] = \sum_{\mathbf{x} \in \mathcal{X}} F(\mathbf{x}) \Pr(\mathbf{x} | \mathbf{x}_S)$.

The \shap\ explanation framework draws from Shapley values in cooperative game theory.
Given a particular instance $\mathbf{x}$, it considers features $\mathbf{X}$ to be players in a coalition game: the game of making a prediction for~$\mathbf{x}$.
\shap\ explanations are defined in terms of a set function $v_{F, \mathbf{x}, \Pr} : 2^\mathbf{X} \to \mathbb{R}$. Its purpose is to evaluate the \textit{``value''} of each coalition of players/features $\mathbf{X}_{S}\subseteq \mathbf{X}$ in making the prediction $F(\mathbf{x})$ under data distribution $\Pr$.
Concretely, following~\citet{DBLP:conf/nips/LundbergL17}, this value function is the conditional expectation of function $F$:
\begin{align}
\vf_{F, \mathbf{x}, \Pr}(\mathbf{X}_S) \defeq \E_{\Pr}[F | \mathbf{x}_S].
\label{eq:v}
\end{align}
We will elide $F$, $\mathbf{x}$, and $\Pr$ when they are clear from context.

Our goal, however, is to assign credit to individual features. In the context of a coalition $\mathbf{X}_{S}$, the {\em contribution} of an individual feature $X \notin \mathbf{X}_{S}$ is given by
\begin{align}
\cf(X, \mathbf{X}_{S}) \defeq \vf(\mathbf{X}_{S}\cup\set{X}) -  \vf(\mathbf{X}_{S}) \label{eq:c}.
\end{align}
where each term is implicitly w.r.t.\ the same $F$, $\mathbf{x}$, and $\Pr$.

Finally, the \shap\ explanation computes a score for each feature
$X \in {\mathbf{X}}$ averaged over all possible contexts,
and thus measures the influence feature $X$ has on the outcome.
Let $\pi$ be a permutation on the set of features $\mathbf{X}$, i.e., $\pi$ fixes a total order on all features.
Let $\pi^{<X}$ be the set of features that come before $X$ in
the order $\pi$. The \shap\ explanations are then defined as computing the following scores.
\begin{defn}[\shap\ Score]
  Fix an entity $\mathbf{x}$, a predictive function $F$, and a data distribution $\Pr$.
  The \shap\ explanation of a feature $X$ is the contribution of $X$ given the features $\pi^{<X}$, averaged over all permutations~$\pi$:
  \begin{align}\label{eq:shap}
    \shap(X) \defeq \frac{1}{n!} \sum_\pi \cf(X, \pi^{<X}).
  \end{align}
\end{defn}

We mention two simple properties of the \shap\ explanations here; for more
discussion see~\citet{DBLP:conf/sp/DattaSZ16} and \citet{lundberg-2020}.
First, for the linear combination of functions $G(.) = \sum_k \lambda_k F_k(.)$, we have that
\begin{align}
\shap_{G}\left(X\right) \ = & \sum_k \lambda_k \shap_{F_k}(X).
\label{eq:shap:linear}
\end{align}
Second, the sum of the \shap\ explanation of all features is
related to the expected value of function $F$:
\begin{align}
\sum_i \shap_F(X_i) \ = & \ F(\mathbf{x}) - \E[F]. \label{eq:shap:sum}
\end{align}

\subsection{Computational Problems}

\label{subsec:computational:problems}

This paper studies the complexity of computing $\shap(X)$; a task we formally define next.
We write $\ttf$ for a class of functions.
We also write $\pbsp_n$ for a class of data distributions over $n$ features, and let $\pbsp = \bigcup_n \pbsp_n$.
We assume that all parameters are rationals.
Because \shap\ explanations are for an arbitrary fixed instance $\mathbf{x}$, we will simplify the notation throughout this paper by assuming it to be the instance $\e=(1,1,\ldots,1)$, and that each domain contains the value~$1$, which is without loss of generality.

\begin{defn}[\shap\ Computational Problems]
  For each function class $\ttf$ and distribution class $\pbsp$, consider the following computational problems.
  \begin{itemize}

  \item[--] The {\em functional \shap\ problem} $\fshap(\ttf, \pbsp)$: given a data distribution $\Pr \in \pbsp$ and a function $F \in \ttf$, compute
  $\shap(X_1), \ldots, \shap(X_n)$.

  \item[--] The {\em decision \shap\ problem} $\dshap(\ttf, \pbsp)$: given a data distribution $\Pr \in \pbsp$, a function $F \in \ttf$, a feature $X \in \mathbf{X}$, and a threshold $t \in \R$, decide if $\shap(X) > t$.
  \end{itemize}

\end{defn}

To establish the complexities of these problems, we use standard
notions of reductions.  A {\em polynomial time reduction} from a
problem $\tta$ to a problem $\ttb$, denoted by $\tta~\leq^P~\ttb$, and
also called a {\em Cook reduction}, is a polynomial-time algorithm for
the problem $\tta$ with access to an oracle for the problem~$\ttb$.
We write $\tta \equiv^P \ttb$ when both $\tta \leq^P \ttb$ and
$\ttb \leq^P \tta$.

In the remainder of this paper will study the computational complexity of these problems for natural hypothesis classes~$\ttf$ that are popular in machine learning, as well as common classes of data distributions $\pbsp$, including those most often used to compute \shap\ explanations.

%% file: sec4-wmc.tex
\section{\shap\ over Fully-Factorized Distributions}
\label{sec:positive-result}

We start our study of the complexity of $\shap$ by considering the
simplest probability distribution: a fully-factorized distribution,
where all features are independent.

There are both practical and computational reasons why it makes sense to assume a fully-factorized data distribution when computing $\shap$ explanations.
First, functions $F$ are often the product of a supervised learning algorithm that does not have
access to a generative model of the data -- it is purely discriminative.
Hence, it is convenient to make the practical assumption that the data
distribution is fully factorized, and therefore easy to estimate.
Second, fully-factorized distributions are highly tractable; for example they make it easy to compute expectations of linear regression functions~\citep{DBLP:conf/ijcai/KhosraviLCB19} and other hard inference tasks~\citep{tutorial-pc}.

\citet{DBLP:conf/nips/LundbergL17} indeed observe that computing the $\shap$-explanation on an arbitrary data distribution is challenging and consider using fully-factorized
distributions (Sec.~4, Eq.~11).  Other prior work on computing
explanations also use fully-factorized distributions of features, e.g.,
\citet{DBLP:conf/sp/DattaSZ16,strumbelj2014}.
As we will show, the $\shap$ explanation can be computed efficiently for several popular classifiers when the distribution is fully factorized.
Yet, such simple data distributions are not a guarantee for tractability: computing \shap\ scores will be intractable for some other common classifiers.

\subsection{Equivalence to Computing Expectations}

Before studying various function classes, we prove a key result that connects the complexity of \shap\ explanations to the complexity of computing expectations.

Let $\indep_n$ be the class of fully-factorized probability
distributions over $n$ discrete and independent random variables
$X_1, \ldots, X_n$. That is, for every instance
$(x_1,\ldots,x_n) \in \mathcal{X}$, we have that
$\pr(X_1=x_1, \ldots, X_n=x_n) = \prod_i \pr(X_i=x_i)$.
Let $\indep \defeq \bigcup_{n \geq 0} \indep_n$.
We show that for every function class $\ttf$, the complexity of $\fshap(\ttf,\indep)$ is
the same as that of the {\em fully-factorized expectation problem}.
\begin{defn}[Fully-Factorized Expectation
  Problem] \label{def:fully:factorized:expectation:problem}
Let $\ttf$ be a class of real-valued
  functions with discrete inputs.  The {\em fully-factorized expectation
    problem} for $\ttf$, denoted $\expprob(\ttf)$, is the following:
  given a function $F \in \ttf$ and a probability distribution $\pr \in \indep$,
  compute $\E_{\pr}(F)$.
\end{defn}

We know from Equation~\ref{eq:shap:sum} that for any function~$F$ over $n$ features,
$\expprob(\set{F}) \leq^P \fshap(\set{F},\indep_n)$,
because $\E[F] = F(\mathbf{x})-\sum_{i=1,n}\shap_F(X_i)$.
In this section we prove that the converse
holds too:
\begin{thm} \label{thm:shap:to:wmc-general}
For any function
$F : \mathcal{X} \to \R$,
we have that
$\fshap(\set{F},\indep_n) \equiv^P \expprob(\set{F}).$
\end{thm}
In other words, for {\em any} function $F$, the complexity of
computing the $\shap$ scores is the same as the complexity of
computing the expected value $\E[F]$ under a fully-factorized data
distribution.  One direction of the proof is immediate:
$\expprob(\set{F}) \leq^P \fshap(\set{F},\indep_n)$ because, if we are
given an oracle to compute $\shap_F(X_i)$ for every feature $X_i$,
then we can obtain $\E[F]$ from Equation~\ref{eq:shap:sum} (recall
that we assumed that $F(\mathbf{x})$ is computable in polynomial
time).  The hard part of the proof is the opposite direction: we will
show in Sec.~\ref{s:equivproof} how to compute $\shap_F(X_i)$ given an
oracle for computing $\E[F]$.  Theorem~\ref{thm:shap:to:wmc-general}
immediately extends to classes of functions $\ttf$, and to any number
of variables, and therefore implies that
$\fshap(\ttf, \indep) \equiv^P \expprob(\ttf)$.

Sections~\ref{ss:indep:tractable} and~\ref{ss:indep:intractable} will discuss the consequences of this result, by delineating which function classes support tractable \shap\ explanations, and which do not. The next section is devoted to proving our main technical result.

\subsection{Proof of Theorem~\ref{thm:shap:to:wmc-general}}
\label{s:equivproof}

We start with the special case when all features $X$ are binary: $\dom(X)=\set{0,1}$.
We denote by $\indeptwo_n$ the class of fully-factorized distributions over binary domains.
\begin{thm} \label{thm:shap:to:wmc} For any function
  $F : \set{0,1}^n \to \R$,
  we have that
  $\fshap(\set{F},\indeptwo_n) \equiv^P \expprob(\set{F})$.
\end{thm}

\begin{proof} We prove only
  $\fshap(F,\indeptwo_n) \leq^P \expprob(\set{F})$; the opposite
  direction follows immediately from Equation~\ref{eq:shap:sum}.  We
  will assume w.l.o.g. that $F$ has $n+1$ binary features
  $\mathbf{X^\prime} = \set{X_0} \cup \mathbf{X}$ and show how to
  compute $\shap_F(X_0)$ using repeated calls to an oracle for
  computing $\E[F]$, i.e., the expectation of the same function $F$,
  but over fully-factorized distributions with different
  probabilities.
  The probability distribution $\pr$ is given to us by $n+1$ rational
  numbers, $p_i \defeq \pr(X_i\!=\!1)$, $i=0,n$; obviously,
  $\pr(X_i\!=\!0)=1-p_i$.
  Recall that the instance whose outcome we want to explain is $\e=(1,\ldots,1)$.
  Recall that for any set $S \subseteq [n]$ we
  write $\e_{S}$ for the event $\bigwedge_{i \in S} (X_i=1)$.  Then, we have that
  \begin{align}
    \shap(X_0) &= \sum_{k=0,n} \frac{k!(n-k)!}{(n+1)!} \, D_k, \ \ \ \ \mbox{where}\label{eq:shap:alt}\\
    D_k & \defeq  \sum_{S \subseteq [n]: |S|=k} \left( \E\!\left[F|\e_{S \cup \set{0}}\right] - \E\!\left[F|\e_{S}\right] \right).\nonumber
  \end{align}
  Let $F_0 \defeq F[X_0:=0]$ and $F_1 \defeq F[X_0:=1]$ (both
  are functions in $n$ binary features, $\mathbf{X} = \set{X_1, \ldots, X_n}$). Then:
  \begin{align*}
  \E\left[F\middle|\e_{S \cup \set{0}}\right] &= \E[F_1|\e_{S}]\\
  \E[F|\e_{S}] &= \E[F_0|\e_{S}] \cdot (1-p_0) + \E[F_1|\e_{S}] \cdot p_0
  \end{align*}
  and therefore $D_k$ is given by:
  \begin{align*}
   D_k = (1-p_0)\sum_{S \subseteq [n]: |S|=k}\left(\E[F_1|\e_S]-\E[F_0|\e_S]\right)
  \end{align*}

  For any function $G$, Equation~\ref{eq:v} defines value $\vf_{G, \e, \pr}(\mathbf{X}_S)$ as $\E[G | \e_{S}]$.  Abusing notation, we write $v_{G,k}$ for the sum of these quantities over all sets $S$ of cardinality~$k$:
  \begin{align}
    v_{G,k} \defeq & \sum_{S \subseteq [n], |S| = k} \E[G|\e_{S}]. \label{eq:v:g:k}
  \end{align}

  We will prove the following claim.
  \begin{claim} \label{claim:thm:shap:to:wmc} Let $G$ be a function
    over $n$ binary variables.  Then the $n+1$ quantities
    $v_{G,0}$ until $v_{G,n}$ can be computed in polynomial time, using $n+1$
    calls to an oracle for $\expprob(\set{G})$.
  \end{claim}
  Note that an oracle for $\expprob(\set{F})$ is also an oracle for both $\expprob(\set{F_0})$ and $\expprob(\set{F_1})$, by simply setting $\pr(X_0=1)=0$ or $\pr(X_0=1)=1$
  respectively.
  Therefore, Claim~\ref{claim:thm:shap:to:wmc} proves Theorem~\ref{thm:shap:to:wmc}, by applying it once to $F_0$
  and once to $F_1$ in order to derive all the quantities $v_{F_0,k}$ and $v_{F_1,k}$, thereby computing $D_k$, and finally computing $\shap_F(X_0)$ using
  Equation~\ref{eq:shap:alt}.
  It remains to prove Claim~\ref{claim:thm:shap:to:wmc}.

  Fix a function $G$ over $n$ binary variables and let $v_k = v_{G,k}$. Let $p_j = \pr(X_j\!=\!1)$, for $j=1,n$, define the distribution over
  which we need to compute $v_0,\ldots,v_n$.
  We will prove the following additional claim.
  \begin{claim}
  \label{claim:thm:shap:to:wmc:identity}
  Given any real number
  $z > 0$, consider the distribution
  $\pr_z(X_j) = p'_j \defeq \frac{p_j+z}{1+z}$, for $j=1,n$. Let $\E_{z}[G]$ denote $\E[G]$ under distribution $\pr_z$. We then have that
  \begin{align}
    \sum_{k=0,n} z^k \cdot v_k = & (1+z)^n \cdot \E_{z}[G]. \label{eq:v:vs:e}
  \end{align}
  \end{claim}
  Assuming Claim~\ref{claim:thm:shap:to:wmc:identity} holds, we prove Claim~\ref{claim:thm:shap:to:wmc}.  Choose any $n+1$
  distinct values for $z$, use the oracle to compute the quantities
  $\E_{z_0}[G], \ldots, \E_{z_n}[G]$, and form a system of $n+1$ linear equations \eqref{eq:v:vs:e} with unknowns $v_0, \ldots, v_n$. Next, observe that its matrix is a non-singular Vandermonde matrix, hence
  the system has a unique solution which can be computed in polynomial
  time.
It remains to prove Claim~\ref{claim:thm:shap:to:wmc:identity}.

   Because of independence, the probability of instance $\mathbf{x} \in \set{0,1}^n$ is
  $\pr(\mathbf{x})= \prod_{i:\mathbf{x}_i=1}p_i \cdot
  \prod_{i:\mathbf{x}_i=0}(1-p_i)$,
  where $\mathbf{x}_i$ looks up the value of feature $X_i$ in instance $\mathbf{x}$.
  Similarly, $\pr_z(\mathbf{x})= \prod_{i:\mathbf{x}_i=1}p_i' \cdot
  \prod_{i:\mathbf{x}_i=0}(1-p_i')$.  Using direct calculations we
  derive:
  \begin{align}
\pr(\mathbf{x})\prod_{i: \mathbf{x}_i=1}\left(1+\frac{z}{p_i}\right)
=     (1+z)^n \cdot \pr_z(\mathbf{x})
\label{eq:pr:prz}
  \end{align}
Separately we also derive the following identity, using the fact that $\pr(\e_{S})=\prod_{i \in S}
p_i$ by independence:
  \begin{align}
\E[G|\e_{S}]=\frac{1}{\prod_{i \in S} p_i} \sum_{\mathbf{x}:
\mathbf{x}_S = \e_S} G(\mathbf{x})\cdot \pr(\mathbf{x})
\label{eq:condexp}
  \end{align}
  We are now in a position to prove Claim~\ref{claim:thm:shap:to:wmc:identity}: %
  \begin{align*}
    \sum_{k=0,n} z^k \cdot v_k &= \sum_{k=0,n} z^k \sum_{S \subseteq [n]: |S|=k} \E[G|\e_S] \\
 &=  \sum_{S \subseteq [n]} z^{|S|} \cdot \E[G|\e_S] \\
&=  \sum_{S \subseteq [n]}\frac{z^{|S|}}{\prod_{i \in S} p_i} \sum_{\mathbf{x}: \mathbf{x}_S = \e_S}G(\mathbf{x}) \cdot \pr(\mathbf{x})
 \end{align*}
The last line follows from Equation~\ref{eq:condexp}.
Next, we simply exchange the summations $\sum_S$ and
  $\sum_\mathbf{x}$, after which we apply the identity
  $\sum_{S \subseteq A} \prod_{i \in S} u_i = \prod_{i \in A}
  (1+u_i)$.
 \begin{align*}
& \textit{(continuing)}\\
= & \sum_{\mathbf{x} \in \set{0,1}^n} G(\mathbf{x}) \cdot \pr(\mathbf{x})  \sum_{S: \mathbf{x}_S = \e_S}\frac{z^{|S|}}{\prod_{i \in S} p_i}\\
= & \sum_{\mathbf{x} \in \set{0,1}^n} G(\mathbf{x}) \cdot \pr(\mathbf{x})\prod_{i: \mathbf{x}_i = 1 }\left(1+\frac{z}{p_i}\right)\\
= & ~(1+z)^n \sum_{\mathbf{x} \in \set{0,1}^n} G(\mathbf{x}) \cdot \pr_z(\mathbf{x}) = (1+z)^n \cdot \E_z[G].
  \end{align*}
The final line uses Equation~\ref{eq:pr:prz}.  This
  completes the proof of Claim~\ref{claim:thm:shap:to:wmc:identity} as well as Theorem~\ref{thm:shap:to:wmc}.
\end{proof}

Next, we generalize this result from binary features to arbitrary
discrete features.  Fix a function with $n$ inputs,
$F: \mathcal{X} (\defeq \prod_i \dom(X_i)) \rightarrow \R$, where each
domain is an arbitrary finite set, $\dom(X_i) = \set{1,2,\ldots,m_i}$;
we assume w.l.o.g. that $m_i > 1$.  A fully factorized probability
space $\pr \in \indep_n$ is defined by numbers $p_{ij} \in [0,1]$,
$i=1,n$, $j=1,m_i$, such that, for all $i$, $\sum_j p_{ij}=1$.  Given
$F$ and $\pr$ over the domain $\prod_i \dom(X_i)$, we define their
{\em projections}, $F_\pi, \pr_\pi$ over the binary domain
$\set{0,1}^n$ as follows.  For any instance
$\mathbf{x} \in \set{0,1}^n$, let $T(\mathbf{x})$ denote the event
asserting that $X_j=1$ iff $\mathbf{x}_j=1$. Formally,
$$T(\mathbf{x}) \defeq \bigwedge_{j: \mathbf{x}_j = 1} (X_j\!=\!1)
\wedge \bigwedge_{j: \mathbf{x}_j = 0} (X_j\!\neq\!1).$$  Then,  the
projections are defined as follows: $  \forall \mathbf{x} \in
\set{0,1}^n$, 
\begin{align}
\pr_\pi(\mathbf{x}) \defeq & \pr(T(\mathbf{x})) &
F_\pi(\mathbf{x}) \defeq & \E[F~|~T(\mathbf{x})] \label{eq:f:pi}
\end{align}
Notice that $F_\pi$ depends both on $F$ and on the probability
distribution $\pr$.  Intuitively, the projections only distinguishes
between $X_j=1$ and $X_j \neq 1$, for example:
\begin{align*}
F_\pi(1,0,0)  = & \E[F | (X_1=1, X_2\neq 1, X_3\neq 1)] \\
\pr_\pi(1,0,0) = & \pr(X_1=1, X_2\neq 1, X_3\neq 1)
\end{align*}

We prove the following result in Appendix~\ref{app:cor:shap:to:wmc}:
\begin{prop} \label{prop:shap:to:wmc} Let $F: \mathcal{X} \to \R$ be a
  function with $n$ input features, and $\pr \in \indep_n$ a fully
  factorized distribution over $\mathcal{X}$.  Then (1) for any
  feature $X_j$, $\shap_{F,\pr}(X_j)=\shap_{F_\pi,\pr_\pi}(X_j)$, and
  (2) $\expprob(\set{F_\pi}) \leq^P \expprob(\set{F})$.
\end{prop}
Item (1) states that the \shap-score of $F$ computed over the
probability space $\pr$ is the same as that of its projection $F_\pi$
(which depends on $\pr$) over the projected probability space
$\pr_\pi$.  Item (2) says that, for any probability space over
$\set{0,1}^n$ (not necessarily $\pr_\pi$), we can compute $\E[F_\pi]$
in polynomial time given access to an oracle for computing $\E[F]$.
We can now complete the proof of
Theorem~\ref{thm:shap:to:wmc-general}, by showing that
$\fshap(\set{F},\indep_n) \leq^P \expprob(\set{F})$.  Given a function
$F$ and probability space $\pr \in \indep_n$, in order to compute
$\shap_{F,\pr}(X_j)$, by item (1) of
Proposition~\ref{prop:shap:to:wmc} it suffices to show how to compute
$\shap_{F_\pi,\pr_\pi}(X_j)$.  By Theorem~\ref{thm:shap:to:wmc}, we
can compute the latter given access to an oracle for computing
$\E[F_\pi]$.  Finally, by item (2) of the proposition, we can compute
$\E[F_\pi]$ given an oracle for computing $\E[F]$.

\subsection{Tractable Function Classes}
\label{ss:indep:tractable}

Given the polynomial-time equivalence between computing \shap\ explanations and computing expectations under fully-factorized distributions,
a natural next question is: which real-world hypothesis classes in machine learning support efficient computation of \shap\ scores?
\begin{cor} \label{cor:tractable-cases}
For the following function classes $\ttf$,
computing \shap\ scores
$\fshap(\ttf,\indep)$ is in polynomial time in the size of the representations of function $F \in \ttf$ and fully-factorized distribution $\pr \in \indep$.
\begin{enumerate}
    \item Linear regression models
    \item Decision and regression trees
    \item Random forests or additive tree ensembles
    \item Factorization machines, regression circuits
    \item Boolean functions in d-DNNF, binary decision diagrams
    \item Bounded-treewidth Boolean functions in CNF
\end{enumerate}
\end{cor}
These are all consequences of Theorem~\ref{thm:shap:to:wmc-general}, and the fact that computing fully-factorized expectations $\expprob(\ttf)$ for these function classes $\ttf$ is in polynomial time. Concretely, we have the following observations about fully-factorized expectations:
\begin{enumerate}
    \item Expectations of linear regression functions are efficiently computed by mean imputation \citep{DBLP:conf/ijcai/KhosraviLCB19}. The tractability of \shap\ on linear regression models is well known. In fact, \citet{strumbelj2014} provide a closed-form formula for this case.
    \item Paths from root to leaf in a decision or regression tree are mutually exclusive. Their expected value is therefore the sum of expected values of each path, which are tractable to compute within $\indep$; see \citet{KhosraviArtemiss20}.
    \item Additive mixtures of trees, as obtained through bagging or boosting, are tractable, by the linearity of expectation.
    \item Factorization machines extend linear regression models with feature interaction terms and factorize the parameters of the higher-order terms~\citep{rendle2010factorization}. Their expectations remain easy to compute.
    Regression circuits are a graph-based generalization of linear regression. \citet{KhosraviNeurIPS19} provide an algorithm to efficiently take their expectation w.r.t.\ a probabilistic circuit distribution, which is trivial to construct for the fully-factorized case.
\end{enumerate}

The remaining tractable cases are Boolean functions. Computing fully-factorized expectations of Boolean functions is widely known as the \textit{weighted model counting} task (WMC) \citep{sang2005performing,chavira2008probabilistic}. WMC has been
extensively studied both in the theory and the AI communities, and the
precise complexity of $\expprob(\ttf)$ is known for many families of
Boolean functions $\ttf$.  These results immediately carry over to the
$\fshap(\ttf,\indep)$ problem through Theorem~\ref{thm:shap:to:wmc-general}:
\begin{itemize}
\item[5.] Expectations can be computed in time linear in the size of various circuit representations, called d-DNNF, which includes binary decision diagrams (OBDD, FBDD) and SDDs \citep{bryant1986graph,darwiche2002knowledge}.\footnote{In contemporaneous work, \citet{arenas2020tractability} also show that the \shap\ explanation is tractable for d-DNNFs, but for the more restricted class of uniform data distributions.}

\item[6.] Bounded-treewidth CNFs are efficiently compiled into OBDD circuits~\citep{ferrara2005treewidth}, and thus enjoy tractable expectations.
\end{itemize}

To conclude this section, the reader may wonder about the algorithmic complexity of solving $\fshap(\ttf,\indep)$ with an oracle for $\expprob(\ttf)$ under the reduction in Section~\ref{s:equivproof}. Briefly, we require a linear number of calls to the oracle, as well as time in $O(n^3)$ for solving a system of linear equations.
Hence, for those classes, such as d-DNNF circuits, where expectations are linear in the size of the (circuit) representation of $F$, computing $\fshap(\ttf,\indep)$ is also linear in the representation size and polynomial in $n$.

\subsection{Intractable Function Classes}
\label{ss:indep:intractable}

The polynomial-time equivalence of Theorem~\ref{thm:shap:to:wmc-general} also implies that computing \shap\ scores must be intractable whenever computing fully-factorized expectations is intractable.
This section reviews some of those function classes $\ttf$, including some for which the computational hardness of  $\expprob(\ttf)$ is well known. We begin, however, with a more surprising result.

Logistic regression is one of the simplest and most widely used machine learning models, yet it is conspicuously absent from Corollary~\ref{cor:tractable-cases}.
We prove that computing the expectation of a logistic regression model is \#P-hard, even under a uniform data distribution, which is of independent interest.

A {\em logistic regression} model is a parameterized function $F(\bm x)\defeq\sigma(\bm w \cdot \bm x)$, where
$\bm w = (w_0, w_1, \ldots, w_n)$ is a vector of weights,
$\sigma(z) = 1 / (1+e^{-z})$ is the logistic function,
$\bm x \defeq (1, x_1, x_2, \ldots, x_n)$, and
$\bm w \cdot \bm x \defeq \sum_{i=0,n} w_i x_i$ is the dot
product. Note that we define the logistic regression function to output probabilities, not data labels.
Let $\logit_n$ denote the class of logistic regression
functions with $n$ variables, and $\logit = \bigcup_n \logit_n$.
We prove the following:
\begin{thm} \label{th:logistic:regression:hard}
    Computing the expectation of a logistic regression model w.r.t.\ a \textit{uniform} data distribution is \#P-hard.
\end{thm}
The full proof in Appendix~\ref{app:proof:logistic:regression:hard} is by reduction from counting solutions to the number partitioning problem.

Because the uniform distribution is contained in $\indep$, and following Theorem~\ref{thm:shap:to:wmc-general}, we immediately obtain:
\begin{cor} \label{cor:logistic:regression:hard}
The computational problems $\expprob(\logit)$ and $\fshap(\logit,\indep)$ are both \#P-hard.
\end{cor}

We are now ready to list general function classes for which computing the \shap\ explanation is \#P-hard.
\begin{cor} \label{cor:intractable-classes}
For the following function classes $\ttf$, computing \shap\ scores
$\fshap(\ttf,\indep)$ is \#P-hard in the size of the representations of function $F \in \ttf$ and fully-factorized distribution $\pr \in \indep$.
\begin{enumerate}
    \item Logistic regression models (Corollary~\ref{cor:logistic:regression:hard})
    \item Neural networks with sigmoid activation functions
    \item Naive Bayes classifiers, logistic circuits
    \item Boolean functions in CNF or DNF
\end{enumerate}
\end{cor}
\noindent Our intractability results stem from these observations:
\begin{enumerate}
    \item[2.] Each neuron is a logistic regression model, and therefore this class subsumes $\logit$.
    \item[3.] The conditional distribution used by a naive Bayes classifier is known to be equivalent to a logistic regression model~\citep{ng2002discriminative}. Logistic circuits are a graph-based classification model that subsumes logistic regression~\citep{LiangAAAI19}.
    \item[4.] For general CNFs and DNFs, weighted model counting, and therefore $\expprob(\ttf)$ is \#P-hard. This is true even for very restricted classes, such as monotone 2CNF and 2DNF functions, and Horn clause logic \citep{wei2005new}.
\end{enumerate}

%% file: sec5-beyondindep.tex
\section{Beyond Fully-Factorized Distributions}
\label{sec:beyond-indep}

Features in real-world data distributions are not
independent. In order to capture more realistic assumptions about the data when computing \shap\ scores, one needs a more intricate probabilistic model. In this section we prove that the
complexity of computing the $\shap$-explanation quickly becomes
intractable, even over the simplest probabilistic models, namely naive
Bayes models. To make computing the $\shap$-explanation as easy as
possible, we will assume that the function $F$ simply outputs the value of one feature. We show that even in this case, even for function classes that are tractable under fully-factorized distributions, computing $\shap$ explanations becomes computationally hard.

Let $\nbn_n$ denote the family of naive Bayes networks over $n+1$
variables $\mathbf{X} = \set{X_0, X_1, \ldots, X_n}$, with binary domains, where $X_0$ is
a parent of all features:
\begin{align*}
  \pr(\mathbf{X}) =  \pr(X_0)\cdot \prod_{i=1,n} \pr(X_i|X_0).
\end{align*}
As usual, the class
$\nbn\defeq \bigcup_{n \geq 0} \nbn_n$.
We write $X_0$ for the function $F$ that returns the value of feature $X_0$; that is, $F(\mathbf{x})=\mathbf{x}_0$.   
We prove the following.
\begin{thm} \label{th:nbn:hard}
  The decision problem $\dshap(\set{X_0}, \nbn)$ is NP-hard.
\end{thm}

The proof in  Appendix~\ref{app:th:nbn:hard} is by reduction from the number partitioning problem, similar to the proof of
Corollary~\ref{cor:logistic:regression:hard}.  
We note that the subset sum problem was also used to prove related hardness results, e.g., for proving hardness of the Shapely value in network games~\cite{DBLP:conf/atal/ElkindGGW08}.

This result is in sharp contrast with the complexity of the
$\shap$ score over fully-factorized distributions in Section~\ref{sec:positive-result}.  
There, the complexity was dictated by the choice of function class $\ttf$.  
Here, the function is as simple as possible, yet computing $\shap$ is hard.  This ruins any hope of
achieving tractability by restricting the function, and this motivates
us to restrict the probability distribution in the next section.  
This result is also surprising because it is efficient to compute marginal probabilities (such as the expectation of $X_0$) and conditional probabilities in naive Bayes distributions.

Theorem~\ref{th:nbn:hard} immediately extends to a large class of probability distributions and functions.
We say that $F$ {\em depends only on $X_i$} if there exist two
constants $c_0 \neq c_1$ such that $F(\mathbf{x}) = c_0$
when $\mathbf{x}_i=0$ and $F(\mathbf{x}) = c_1$ when $\mathbf{x}_i=1$.  In other
words, $F$ ignores all variables other than $X_i$, and {\em does}
depend on $X_i$. We then have the following.
\begin{cor} \label{cor:nbn:hard:consequences} The problem
  $\dshap(\ttf,\pbsp)$ is NP-hard, when $\pbsp$ is any of the
  following classes of distributions:
  \begin{enumerate}
  \item Naive Bayes, bounded-treewidth Bayesian networks
  \item Bayesian networks Markov networks, Factor graphs
  \item Decomposable probabilistic circuits
  \end{enumerate}
  and when $\ttf$ is any class that contains some function $F$ that depends only on $X_0$, including the class of linear regression models and all the classes listed in Corollaries~\ref{cor:tractable-cases} and \ref{cor:intractable-classes}.
\end{cor}
This corollary follows from two simple observations.
First, each of the classes of probability distributions listed in the
corollary can represent a naive Bayes network over binary
variables $\mathbf{X}$.  For example, a Markov network will
consists of $n$ factors $f_1(X_0,X_1), f_2(X_0,X_2), \ldots, f_n(X_0,X_n)$; similar simple arguments prove that all the other classes can represent naive Bayes, including tractable probabilistic circuits such as sum-product networks~\citep{tutorial-pc}.

Second, for each function that depends only on $X_0$, there exist two distinct constants $c_0\neq c_1 \in \R$ such
that $F(\mathbf{x})=c_0$ when $\mathbf{x}_0=0$ and $F(\mathbf{x})=c_1$ when $\mathbf{x}_0=1$.  For example, if
we consider the class of logistic regression functions
$F(\mathbf{x})=\sigma(\sum_i w_i \mathbf{x}_i)$, then we choose the weights $w_0 = 1$,
$w_1=\ldots=w_n=0$ and we obtain $F(\mathbf{x})=1/2$ when $\mathbf{x}_0=0$ and
$F(\mathbf{x})=1/(1+e^{-1})$ when $\mathbf{x}_0=1$.  Then, over the binary domain $\set{0,1}$ the function is equivalent to $F(\mathbf{x})=(c_1-c_0)\mathbf{x}_0+c_0$,
and, therefore, by the linearity of the $\shap$ explanation
(Equation~\ref{eq:shap:linear}) we have
$\shap_F(X_0) = (c_1-c_0) \cdot \shap_{X_0}(X_0)$ (because the
$\shap$ explanation of a constant function $c_0$ is 0) for which, by
Theorem~\ref{th:nbn:hard}, the decision problem is NP-hard.

We end this section by proving that Theorem~\ref{th:nbn:hard}
continues to hold even if the prediction function $F$ is the value
of some leaf node of a (bounded treewidth) Bayesian Network.  In other words, the hardness
of the $\shap$ explanation is not tied to the function returning
the root of the network, and applies to more general functions.  
\begin{cor} \label{cor:nbn:hard} The $\shap$ decision problem for Bayesian networks with latent variables is NP-hard, even if the function $F$ returns a single leaf variable of the network.
\end{cor}
\noindent The full proof is given in Appendix~\ref{app:leaf-hard}.

%% file: sec6-empirical.tex
\section{\shap\ on Empirical Distributions}
\label{sec:empirical}

In supervised learning one does not require a generative model of the
data, instead, the model is trained on some concrete data set: the
{\em training data}.  When some probabilistic model is needed, then
the training data itself is conveniently used as a probability model,
called the {\em empirical distribution}.  This distribution captures
dependencies between features, while its set of possible worlds is limited to those in the data set.  For example, the intent of the KernelSHAP
algorithm by~\citet{DBLP:conf/nips/LundbergL17} is to compute the
$\shap$ explanation on the empirical distribution.
In another example,~\citet{aas2019} extend KernelSHAP to work with dependent features, by estimating the conditional probabilities from the empirical distribution.

Compared to the data distributions considered in the previous sections, the empirical distribution has one key advantage: it has many fewer possible worlds with positive probability -- this suggests increased tractability.
Unfortunately, in this section, we prove that computing the $\shap$ explanation
over the empirical distribution is \#P-hard in general.

To simplify the presentation, this section assumes that all
features are binary: $\dom(X_j)=\set{0,1}$.  The probability distribution is
given by a 0/1-matrix $\bm d = (x_{ij})_{i \in [m], j \in [n]}$, where
each row $(x_{i1},\ldots, x_{in})$ is an outcome with probability
$1/m$.  One can think of $\bm d$ as a dataset with $n$ features and
$m$ data instances, where each row $(x_{i1}, \ldots, x_{in})$ is one data
instance.  Repeated rows are possible: if a row occurs $k$ times, then its
probability is $k/m$.  We denote by $\ttx$ the class of empirical
distributions.
The predictive function can be any function
$F : \set{0,1}^n \rightarrow \R$.  As our data distribution is no longer strictly positive, we adopt the standard convention
that $\E[F | \mathbf{X}_S=1] =0$ when $\pr(\mathbf{X}_S=1)=0$.

Recall from Section~\ref{subsec:computational:problems} that, by
convention, we compute the $\shap$-explanation w.r.t. instance
$\e=(1,1,\ldots,1)$, which is without loss of generality.
Somewhat surprisingly, the complexity of computing the
$\shap$-explanation of a function $F$ over the empirical distribution
given by a matrix $\bm d$ is related to the problem of computing the
expectation of a certain CNF formula associated to $\bm d$.
\begin{defn} \label{def:phi:matrix} The {\em positive, partitioned 2CNF formula}, PP2CNF,
  associated to a matrix $\bm d \in \set{0,1}^{m \times n}$ is:
  \begin{align*}
      \Phi_{\bm d} \defeq \bigwedge_{(i,j): x_{ij}=0} (U_i \vee V_j).
  \end{align*}
\end{defn}

Thus, a PP2CNF formula is over $m+n$ variables
$U_1,\ldots,U_m,V_1,\ldots,V_n$, and has only positive clauses.  The
matrix $\bm d$ dictates which clauses are present.  A {\em
  quasy-symmetric probability distribution} is a fully factorized
distribution over the $m+n$ variables for which there exists two
numbers $p,q \in [0,1]$ such that for every $i=1,m$, $\pr(U_i=1)=p$ or
$\pr(U_i=1)=1$, and for every $j=1,n$, $\pr(V_j=1)=q$ or $\pr(V_j=1)$.
In other words, all variables $U_1, \ldots, U_m$ have the same
probability $p$, or have probability $1$, and similarly for the
variables $V_1, \ldots, V_n$.  We denote by $\expquasiprob(\ppcnf)$
the expectation computation problem for PP2CNF over quasi-symmetric
probability distributions.  $\expquasiprob(\ppcnf)$ is \#P-hard,
because computing $\E[\Phi_{\bm d}]$ under the uniform distribution
(i.e.  $\pr(U_1=1)=\cdots=\pr(V_n=1)=1/2$) is
\#P-hard~\cite{DBLP:journals/siamcomp/ProvanB83}.  We prove:
\begin{thm} \label{th:empirical:hard}
    Let $\ttx$ be the class of empirical distributions,
    and $\ttf$ be any class of function such that, for each $i$, it
    includes some function that depends only on $X_i$.
    Then, we have that $\fshap(\ttf, \ttx) \equiv^P \expquasiprob(\ppcnf)$.

    As a consequence, the problem $\fshap(\ttf, \ttx)$ is \#P-hard in the size of the
    empirical distribution.
\end{thm}

The theorem is surprising, because the set of possible outcomes of an
empirical distribution is small.  This is unlike all the
distributions discussed earlier, for example those mentioned in
Corollary~\ref{cor:nbn:hard:consequences}, which have $2^n$
possible outcomes, where $n$ is the number of features.  In
particular, given an empirical distribution $\bm d$, one can compute the expectation
$\E[F]$ in polynomial time for any function $F$, by doing just one
iteration over the data.  Yet, computing the $\shap$ explanation of
$F$ is \#P-hard.

Theorem~\ref{th:empirical:hard} implies hardness of \shap\ explanations on the empirical distribution for a large class of functions.
\begin{cor}
  The problem  $\fshap(\ttf, \ttx)$ is \#P-hard, when $\ttx$ is
  the class of empirical distributions, and $\ttf$ is any class such that for each feature $X_i$, the class contains some function that depends only on $X_i$. This includes all the function classes listed in Corollaries~\ref{cor:tractable-cases} and \ref{cor:intractable-classes}.
\end{cor}
For instance, any class of Boolean function that contains the $n$
single-variable functions $F \defeq X_i$, for $i=1,n$, fall under
this~corollary.  Section~\ref{sec:beyond-indep} showed an example of
how the class of logistic regression functions fall under this
corollary as~well.

The proof of Theorem~\ref{th:empirical:hard} follows from the following
technical lemma, which is of independent interest:

\begin{lmm} \label{lemma:empirical:hard} We have that:
  \begin{enumerate}
  \item \label{item:empirical:hard:1} For every matrix $\bm d$,
    $\fshap(\ttf, \bm d) \leq^P \expquasiprob(\set{\Phi_{\bm d}})$.
  \item \label{item:empirical:hard:2}
    $\expquasiprob(\ppcnf) \leq^P \fshap(\ttf, \ttx)$.
  \end{enumerate}
\end{lmm}

The proof of the Lemma is given in
Appendix~\ref{sec:proof:empirical:hard:1}
and~\ref{sec:proof:empirical:hard:2}.
The first item says that we can compute the $\shap$-explanation in
polynomial time using an oracle for computing $\E[\Phi_{\bm d}]$ over
quasi-symmetric distributions.  The oracle is called only on the
PP2CNF $\Phi_{\bm d}$ associated to the data $\bm d$, but may perform
repeated calls, with different probabilities of the Boolean variables.
This is somewhat surprising because the $\shap$ explanation is over an
empirical distribution, while $\E[\Phi_{\bm d}]$ is taken over a
fully-factorized distribution; there is no connection between these
two distributions.  This item immediately implies
$\fshap(\ttf, \ttx) \leq^P \expquasiprob(\ppcnf)$, where $\ttx$ is the
class of empirical distributions $\bm d$, since the formula
$\Phi_{\bm d}$ is in the class $\ppcnf$.

The second item says that a weak form of converse also holds. It
states that we can compute in polynomial time the expectation
$\E[\Phi]$ over a quasi-symmetric probability distributions by using
an oracle for computing $\shap$ explanations, over several matrices
$\bm d$, but not necessarily restricted to the matrix associated to
$\Phi$.  Together, the two items of the lemma prove
Theorem~\ref{th:empirical:hard}.

We end this section with a comment on the TreeSHAP algorithm
in~\citet{lundberg-2020}, which is computed over a distribution
defined by a tree-based model.  Our result implies that the problem that
TreeSHAP tries to solve is \#P-hard.  This follows immediately by
observing that every empirical distribution $\bm d$ can be represented
by a binary tree of size polynomial in the size of $\bm d$.  The tree
examines the attributes in the order $X_1, X_2, \ldots, X_n$, and each
decision node for $X_i$ has two branches: $X_i=0$ and $X_i=1$.  A
branch that does not exists in the matrix $\bm d$ will end in a leaf
with label $0$.  A complete branch that corresponds to a row
$x_{i1}, x_{i2}, \ldots, x_{in}$ in $\bm d$ ends in a leaf with label
$1/m$ (or $k/m$ if that row occurs $k$ times in $\bm d$).  The size of
this tree is no larger than twice the size of the matrix (because of
the extra dead end branches).
This concludes our study of \shap\ explanations on the empirical distribution.

%% file: sec7-conclusion.tex
\section{Perspectives and Conclusions}

We establish the complexity of computing the \shap\ explanation in three important
settings. First, we consider fully-factorized data distributions and show that
for any prediction model, the complexity of computing the \shap\ explanation is the
same as the complexity of computing the expected value of the model. It follows
that there are commonly used models, such as logistic regression, for which
computing \shap\ explanations is intractable. Going beyond fully-factorized
distributions, we show that computing \shap\ explanations is also intractable for simple functions and simple distributions -- naive Bayes and empirical distributions.

The recent literature on \shap\ explanations predominantly studies tradeoffs of
variants of the original \shap\ formulation, and relies on approximation
algorithms to compute the explanations. These approximation algorithms, however,
tend to make simplifying assumptions which can lead to counter-intuitive
explanations, see e.g., \citet{advlime:aies20}. We believe that more focus
should be given to the computational complexity of \shap\ explanations.  In
particular, which classes of machine learning models can be explained
efficiently using the \shap\ scores? Our results show that, under
the assumption of fully-factorized data distributions, there are classes of
models for which the \shap\ explanations can be computed in polynomial time. In future work, we plan to explore if there are classes of models for which the complexity of the \shap\ explanations is tractable under more complex data distributions, such as the ones defined by tractable probabilistic circuits~\cite{tutorial-pc} or tractable symmetric probability spaces~\citep{DBLP:conf/kr/BroeckMD14,DBLP:conf/pods/BeameBGS15}.

%% file: secA.tex
\section{Discussion on  the TreeSHAP algorithm}
\label{sec:bug}

\begin{algorithm}[t]
  \caption{Algorithm to compute value function $\vf$ from \cite{treeshap2018}}
  \label{alg:treeshap:bug}

  \newcommand{\TAB}{\hspace*{1em}}
  \newcommand{\STAB}{\hspace*{0.5em}}
  \newcommand{\IF}{\textbf{if}\STAB}
  \newcommand{\THEN}{\textbf{then}}
  \newcommand{\ELSE}{\textbf{else}\STAB}
  \newcommand{\RETURN}{\textbf{return}\STAB}

  \textbf{procedure} EXPVALUE(x, S, tree = \{$v,a,b,t,r,d$\}) \\[0.3em]
  \TAB\textbf{procedure} G(j)\\
  \TAB\TAB\IF $v_j \neq $ \emph{internal} \THEN \\
  \TAB\TAB\TAB \RETURN $v_j$ \\
  \TAB\TAB \ELSE \\
  \TAB\TAB\TAB\IF $d_j \subseteq S$\STAB \THEN \\
  \TAB\TAB\TAB\TAB \RETURN $G(a_j)$ \IF $x_{d_j} \leq t_j$ \ELSE  $G(b_j)$ \\
  \TAB\TAB\TAB \ELSE \\
  \TAB\TAB\TAB\TAB \RETURN $(G(a_j) \cdot r_{a_j} + G(b_j) \cdot r_{b_j}) / r_j$ \\[0.5em]
  \TAB \RETURN $G(1)$
\end{algorithm}

\citet{treeshap2018} propose TreeSHAP, a variant of \shap\ explanations for
tree-based machine learning models such as decision trees, random forests and
gradient boosted trees. The authors claim that, for the case when both the
model and probability distribution are defined by a tree-based model, the
algorithm can compute the exact \shap\ explanations in polynomial time.
However, it has been pointed out in Github discussions (e.g.,
\url{https://github.com/christophM/interpretable-ml-book/issues/142}) that the
TreeSHAP algorithm does not compute the \shap\ explanation as defined in
Section~\ref{sec:shap}. In this section, we provide a concrete example of this
shortcoming.

The main shortcoming of the TreeSHAP algorithm is captured by
Algorithm~\ref{alg:treeshap:bug}. The authors claim that
Algorithm~\ref{alg:treeshap:bug} computes the conditional expectation
$\E[F \mid \bm x_S]$, for a given set of features $S$ and tree-based model $F$. We
first describe the algorithm and then show by example that this algorithm does
not accurately compute the conditional expectation.

Algorithm~\ref{alg:treeshap:bug} takes as input a feature vector $\bm x$, a set
of features $S$, and a binary tree, which represents the tree-based model. The
tree is defined by the following vectors: $v$ is a vector of node values;
internal nodes are assigned the value \emph{internal}. The vectors $a$ and $b$
represent the left and right node indexes for each internal node. The vector $t$
contains the thresholds for each internal node, and $d$ is a vector of indexes
of the features used for splitting in internal nodes. The vector $r$ represents
the cover of each node (i.e., how many data samples fall in that sub-tree).

The algorithm proceeds recursively in a top-down traversal of the tree. For
inner nodes, the algorithm follows the decision path for $x$ if the split
feature is in $S$, and takes the weighted average of both branches if the split
feature is not in $S$.  For leaf nodes, it returns the value of the node, which
corresponds to the prediction of the model.

The algorithm does not accurately compute the conditional expectation
$E[F \mid \bm x_S]$, because it does not normalize expectation by the
probability of the condition. The following simple example shows that the value
returned by Algorithm~\ref{alg:treeshap:bug} does not represent the conditional
expectation.

\begin{ex}
  We consider the following dataset and decision tree model. The dataset
  has two binary variables $X_1$ and $X_2$, and each instance $(x_1, x_2)$ is
  weighted by the occurrence count (i.e., the instance (0,0) occurs twice in the
  dataset).  We want to compute $E[F(X_1,X_2) | X_2 = 0]$, where $F(X_1,X_2)$ is
  the outcome of the decision tree.
  \begin{center}
    \begin{minipage}{0.3\columnwidth}
      \begin{tabular}{c|c|c}
        $X_1$ & $X_2$ & \# \\\hline
        0 & 0 & 2 \\
        0 & 1 & 1 \\
        1 & 0 & 1 \\
        1 & 1 & 2
      \end{tabular}
    \end{minipage}
    \begin{minipage}{0.6\columnwidth}
      \begin{tikzpicture}[scale=0.8]
        \tikzstyle{tnode}=[draw, rectangle, inner sep = .08cm]

        \node[tnode] at (0,0) (r) {$X_1$};
        \node[tnode] at ($(r)+(-1.6,-1)$) (n1) {$X_2$};
        \node[tnode] at ($(r)+(1.6,-1)$) (n2) {$X_2$};

        \draw (r) -- (n1) node[midway, fill=white, scale=0.9] {0}; 
        \draw (r) -- (n2) node[midway, fill=white, scale=0.9] {1}; 

        \node at ($(n1)+(-.8,-1.5)$) (o1) {$F(0,0)$};
        \node at ($(n1)+(.8,-1.5)$) (o2) {$F(0,1)$};

        \node at ($(n2)+(-.75,-1.5)$) (o3) {$F(1,0)$};
        \node at ($(n2)+(.75,-1.5)$) (o4) {$F(1,1)$};

        \draw (n1) -- (o1) node [midway, fill=white, scale=0.9] {0}; 
        \draw (n1) -- (o2) node [midway, fill=white, scale=0.9] {1};

        \draw (n2) -- (o3) node[midway, fill=white, scale=0.9] {0}; 
        \draw (n2) -- (o4) node[midway, fill=white, scale=0.9] {1};
      \end{tikzpicture}
    \end{minipage}
  \end{center}

  The correct value is:
  $$E[F(X_1,X_2) \mid X_2 = 0] = 2/3 \cdot F(0,0) + 1/3 \cdot F(1,0)$$
  This is because there are three items with $X_2=0$, and their probabilities are $2/3$
  and $1/3$. 

  Algorithm~\ref{alg:treeshap:bug}, however, returns:
  $$G(1) =  1/2 \cdot F(0,0) + 1/2 \cdot F(1,0),$$
  and thus does not compute $E[F(X_1,X_2) \mid X_2 = 0]$.

\end{ex}

%% file: secC.tex
\section{Proof of Proposition~\ref{prop:shap:to:wmc}}

\label{app:cor:shap:to:wmc}

We start with item (1).  Recall that
$\dom(X_i)= \set{1,2,3,\ldots,m_i}$.  We denote by
$p_{i1},p_{i2},\ldots, p_{im_i}$ their probabilities, thus
$\sum_{j=1,m_i} p_{ij}=1$.  By definition, the projected distribution
is: $\pr_\pi(X_i=1) \defeq p_{i1}$, and $\pr_\pi(X_i=0)=1-p_{i1}$.  We
denote by $\E_\pi$ be the corresponding expectation.  Our goal is to
prove $\shap_{F,\Pr}(X_j)=\shap_{F_\pi,\Pr_\pi}(X_j)$.

Let $\e_S$ again denote the event $\bigwedge_{i \in S} (X_i = 1)$.
Note that, by construction, for any set $S$,
$\pr(\e_S)=\pr_\pi(\e_S)$.  Recall that for any instance
$\mathbf{x} \in \set{0,1}^n$, we let $T(\mathbf{x})$ denote the event
asserting that $X_j=1$ iff $\mathbf{x}_j=1$; formally,
$$T(\mathbf{x}) \defeq \bigwedge_{j: \mathbf{x}_j = 1} (X_j\!=\!1) \wedge \bigwedge_{j: \mathbf{x}_j = 0} (X_j\!\neq\!1).$$
Then, for any instance $\mathbf{x} \in \set{0,1}^n$,
$$\pr(T(\mathbf{x}))=\prod_{i: \mathbf{x}_i=1} p_{i1}\cdot
\prod_{i: \mathbf{x}_i=0}
(p_{i2}+p_{i3}+\cdots)=\pr_\pi(\mathbf{x}).$$
Thus, there are $2^n$ disjoint events $T(\mathbf{x})$, which partition
the space $\mathcal{X}$.  Therefore, for every set $S$:
\begin{align*}
  \E[F \wedge  \e_S]
  &= \sum_{\mathbf{x} : \forall i \in S, \mathbf{x}_i = 1} \E[F | T(\mathbf{x})]\, \pr(T(\mathbf{x}))\\
  &= \sum_{\mathbf{x} : \forall i \in S, \mathbf{x}_i = 1} F_\pi(\mathbf{x})\, \pr_\pi(\mathbf{x})\\
  &= \E_\pi[F_\pi\wedge \e_S]
\end{align*}
This implies that $\E[F | \e_S] = \E_\pi[F_\pi| \e_S]$ for any set
$S$, and $\shap_{F,\Pr}(X_j)=\shap_{F_\pi, \Pr_\pi}(X_j)$ for all $j$
follows from Equation~\ref{eq:shap:alt}.

We now prove item (2): we show how to compute $\E[F_\pi]$ given an
oracle for computing $\E[F]$.  Recall that we want to compute
$\E[F_\pi]$ on some arbitrary distribution $\pr_\pi'$ on
$\set{0,1}^n$; this should not be confused with the probability
$\pr_\pi$ defined in Eq.\ref{eq:f:pi}, and used to define the function
$F_\pi$.  Denote $q_i \defeq \pr_\pi'(X_i\!=\!1)$, thus
$\pr_\pi(X_i\!=\!0)= 1-q_i$.  To compute $\E[F_\pi]$ we will use the
oracle for computing $\E[F]$, on the probability space defined by the
numbers $p_{ij}'$, $i=1,n$, $j=1,m_i$ defined as follows:
\begin{align*}
  w_i \defeq & \frac{1-q_i}{q_i}  \text{\qquad and \qquad } Z \defeq \prod_{i=1,n} q_i\\
  W \defeq & \prod_{i=1,n}\left(1+\sum_{j=2,m_i} \frac{p_{ij}w_i}{1-p_{i1}}\right)
\\
  p'_{i1} \defeq & \frac{1}{W}  & i = & 1,n\\
  p'_{ij} \defeq & \frac{p_{ij}w_i}{W(1-p_{i1})} & i = & 1,n; j=2,m_i
\end{align*}

One can check that the numbers $p_{ij}'$ indeed define a probability
space on $\mathcal{X}$, in other words $p_{ij}' \in [0,1]$ and, for
all $i=1,n$: $\sum_{j=1,m_j}p_{ij}' = 1$.  We denote by $\pr'$ the
probability space that they define, and denote by $\E'[F]$ the
expectation of $F$ in this space.  We prove:

\begin{claim}
  $\E_\pi[F_\pi] = Z\cdot W\cdot \E'[F]$
\end{claim}

The claim immediately proves item (2) of
Proposition~\ref{prop:shap:to:wmc}: we simply invoke the oracle to
compute $\E'[F]$, then multiply with the quantities $Z$ and $W$, both
of which are computable in polynomial time.  It remains to prove the
claim.

We start with some observations and notations.  Recall that the
projection $F_\pi$ depends on both $F$ and on $\pr$, see
Equation~\ref{eq:f:pi}.  We express it here in terms of the
probabilities $p_{ij}$:
\begin{align*}
  F_\pi[\mathbf{x}] =~ & \E[F | T(\mathbf{x})] = \frac{\E[F \cdot T(\mathbf{x})]}{\pr(T(\mathbf{x}))} \\
    =~ & \frac{\sum_{\tau \in \mathcal{X}: \mathbf{x}^{-1}(1) = \tau^{-1}(1)}F(\tau) \cdot\prod_i p_{i\tau_i}}{\prod_{i: \mathbf{x}_i=1} p_{i1}\cdot \prod_{i:\mathbf{x}_i\neq 1} (1-p_{i1})}\\
    =~ & \sum_{\tau \in \mathcal{X}: \mathbf{x}^{-1}(1) = \tau^{-1}(1)}F(\tau)\cdot \prod_{i: \tau_i\neq 1}\frac{p_{i\tau_i}}{1-p_{i1}}.
\end{align*}
We used the fact that, for every instance
$\mathbf{x} \in \mathcal{X}$,
$\pr(\mathbf{x}) = \prod_i p_{i\mathbf{x}_i}$, and denoted by
$\mathbf{x}^{-1}(1)$ the set of feature indices for which example
$\mathbf{x}$ has value $1$.  We now prove the claim by applying
directly the definition of $\E_\pi[F_\pi]$:
\allowdisplaybreaks
\begin{align*}
  \E_\pi[F_\pi] = & \sum_{\theta \in \set{0,1}^n} F_\pi(\theta) \prod_{i:  \theta_i=1}q_i \prod_{i: \theta_i=0}(1-q_i)\\
  = & Z \cdot \sum_{\theta \in \set{0,1}^n} F_\pi(\theta) \prod_{i:  \theta_i=0}w_i \\
  = & Z \cdot \sum_{\scriptsize \begin{array}{l} \theta \in \set{0,1}^n \\ \tau \in \mathcal{X} \\ \theta^{-1}(1) = \tau^{-1}(1) \end{array}}F(\tau) \prod_{i: \tau_i\neq 1} \frac{p_{i\tau_i}}{1-p_{i1}}\prod_{i: \theta_i=0}w_i\\
  = & Z \cdot \sum_{\tau \in \mathcal{X}}F(\tau) \prod_{i: \tau_i\neq 1} \frac{p_{i\tau_i}w_i}{1-p_{i1}}\\
  = & Z \cdot W \cdot \sum_{\tau \in \mathcal{X}}F(\tau) \prod_i p'_{i\tau_i}\\
  = & Z \cdot W \cdot \E'[F]
\end{align*}
In line 4 we noticed that the conditions $\tau_i\neq 1$ and
$\theta_i=0$ are equivalent, because $\theta^{-1}(1)=\tau^{-1}(1)$,
and that the assignment $\tau \in \mathcal{X}$ uniquely defines
$\theta$, hence $\theta$ can be dropped from the summation.  This
completes the proof of the claim, and of
Proposition~\ref{prop:shap:to:wmc}.

%% file: secD.tex
\section{Proof of Theorem~\ref{th:logistic:regression:hard}}
\label{app:proof:logistic:regression:hard}

The \textit{number partitioning} problem, $\numpar$, is the following: given $n$
natural numbers $k_1, \ldots, k_n$, decide whether there exists a subset
$S \subseteq [n]$ that partitions the numbers into two sets with equal sums:
$\sum_{i \in S} k_i = \sum_{i \not\in S} k_i$.  $\numpar$ is known to be
NP-complete.  The corresponding counting problem, in notation $\#\numpar$, asks
for the number of sets $S$ such that
$\sum_{i \in S} k_i = \sum_{i \not\in S} k_i$, and is \#P-hard.

We show that we can solve the $\#\numpar$ problem using an oracle for
$\E_{\mathbf{U}}[F]$, where $F$ is a logistic regression function and
$\mathbf{U}$ is the uniform probability distribution. This implies that
computing the expectation of a logistic regression function is \#P-hard.

Fix an instance of $\numpar$, $k_1, \ldots, k_n$, and assume
w.l.o.g. that the sum of the numbers $k_i$ is even, $\sum_i k_i = 2c$
for some natural number $c$.  Let 
\begin{align}
P \defeq & \setof{S}{S \subseteq [n], \sum_{i\in S}k_i = c} \label{eq:numpar:p}
\end{align}
For each set $S \subseteq [n]$, denote by $\bar S$ its complement.
Obviously, $S \in P$ iff $\bar S \in P$, therefore $|P|$ is an even
number.

We next describe an algorithm that computes $|P|$ using an oracle for
$\E_\mathbf{U}[F]$, where $F$ is a logistic regression function and $\mathbf{U}$
is the uniform probability distribution.  Let $m$ be a natural number large
enough, to be chosen later, and define the following weights:
\begin{align*}
  w_0 \defeq & -\frac{m}{2} - mc & w_i \defeq & mk_i\;\;\; \forall i=1,n
\end{align*}

Let $\bm w = (w_1, \ldots, w_n)$, then
$F(x_1, \ldots, x_n) \defeq \sigma(w_0 + \bm w \cdot \bm x)$ is the logistic
regression function defined by the weights $w_0,\ldots, w_n$.

\begin{claim} Let $\varepsilon \defeq 1/2^{n+3}$.  If $m$ satisfies
  both $2\sigma(-m/2) \leq \varepsilon$ and
  $1-\sigma(m/2) \leq \varepsilon$, then:
  \begin{align*}
    |P| = &  \left\lceil2^n - \frac{2^{n+1}\E[F]}{1-\varepsilon}\right\rceil
  \end{align*}
\end{claim}

The claim immediately proves the theorem: in order to solve the $\#\numpar$
problem, compute $\E[F]$ and then use the formula above to derive $|P|$.  To
prove the claim, for each set $S \subseteq [n]$ denote by:
\begin{align*}
  \text{weight}(S)  \defeq & -\frac{m}{2} - mc + m(\sum_{i\in S} k_i)
\end{align*}
Let $\mathbf{U}$ denote the uniform probability distribution over the domain $\{0,1\}^n$. Then, 
\begin{align*}
  \E_{\mathbf{U}}[F] 
  &= \frac{1}{2^n} \sum_{\mathbf{x}} \sigma(w_0 + \bm w \cdot \mathbf{x})\\
  &= \frac{1}{2^n} \sum_{\mathbf{x}} \sigma( -\frac{m}{2} - mc + m(\sum_{i\in [n]} k_i \mathbf{x}_i))\\
  &= \frac{1}{2^n} \sum_{\mathbf{x}} \sigma( -\frac{m}{2} - mc + m(\sum_{i : \mathbf{x}_i = 1} k_i))\\
  &= \frac{1}{2^n} \sum_{S \subseteq [n]} \sigma(\text{weight}(S))\\
  &= \frac{1}{2^{n+1}} \sum_{S \subseteq [n]} (\sigma(\text{weight}(S))+\sigma(\text{weight}(\bar S)))
\end{align*}
If $S$ is a solution to the number partitioning problem ($S \in P$), then:
\begin{align*}
\sigma(\text{weight}(S))+\sigma(\text{weight}(\bar S)) = 2\sigma(-m/2)
\end{align*}
Otherwise, one of $\text{weight}(S)$, $\text{weight}(\bar S)$ is
$\geq m/2$ and the other is $\leq -3m/2$ and therefore:
\begin{align*}
  \sigma(m/2) \leq \sigma(\text{weight}(S))+\sigma(\text{weight}(\bar S)) \leq 1+\sigma(-3m/2)
\end{align*}
Since $\varepsilon = 1/2^{n+3}$, and $m$ satisfies
  both $2\sigma(-m/2) \leq \varepsilon$ and
  $1-\sigma(m/2) \leq \varepsilon$,  we have:
\begin{align*}
  S \in P: && 0 \leq\sigma(\text{weight}(S))+\sigma(\text{weight}(\bar S)) & \leq \varepsilon \\
  S \not\in P: &&  1-\varepsilon \leq\sigma(\text{weight}(S))+\sigma(\text{weight}(\bar S))&\leq 1+\varepsilon
\end{align*}

This implies:
\begin{align*}
  \frac{2^n-|P|}{2^{n+1}}\left(1-\varepsilon\right) \leq
  &\ \E[F] \leq \frac{|P|}{2^{n+1}}\,\varepsilon + \frac{2^n-|P|}{2^{n+1}}(1+\varepsilon)\\
  |P| \geq &  2^n - \frac{2^{n+1}\E[F]}{1-\varepsilon} \\
  |P| \leq & 2^n(1+\varepsilon) - 2^{n+1} \E[F]
\end{align*}
Thus, we have a lower and an upper bound for $|P|$.  Since
$\E[F] \leq 1$, the difference between the two bounds is $<1$ and
there exists at most one integer number between them, hence $|P|$
is equal to the ceiling of the lower bound (and also to the floor of
the upper bound), proving the claim.

%% file: secE.tex
\section{Proof of Theorem~\ref{th:nbn:hard}}
\label{app:th:nbn:hard}

We use a reduction from the decision version of number partitioning
problem, $\numpar$, which is NP-complete, see
Sec.~\ref{app:proof:logistic:regression:hard}.

As before we assume w.l.o.g. that the sum of the numbers $k_i$ is
even, $\sum_i k_i = 2c$ for some natural number $c$.  Let $m$ be a
large natural number to be defined later.  We reduce the $\numpar$
problem to the $\dshap(\{X_0\},\nbn)$ problem.  The Naive
Bayes network $\nbn$ consists of $n+1$ binary random variables
$X_0, \ldots, X_n$.  Let $X_i, \bar X_i$ denote the events
$X_i=1$ and respectively $X_i=0$. We define the following probabilities of the \nbn:
\begin{align*}
  \frac{\pr(X_0)}{\pr(\bar X_0)}=&e^{-\frac{m}{2} - mc}
  &\frac{\pr(X_i|X_0)}{\pr(X_i|\bar X_0)}=& e^{mk_i}
\end{align*}
The probabilities $\pr(\bar X_0)$ and $\pr(X_i | \bar X_0)$ can be
chosen arbitrarily (with the obvious constraints
$\pr(\bar X_0) \leq e^{\frac{m}{2}+mc}$ and
$\pr(X_i|\bar X_0) \leq e^{-mk_i}$).  As required, our classifier is
$F(X_0,\ldots, X_n) \defeq X_0$.  Let
$a_k \defeq \frac{k!(n-k)!}{(n+1)!}$ and let $\varepsilon> 0$ be any
number such that $\varepsilon \leq a_k$ for all $k=0,1,\ldots,n$.
We prove:

\begin{claim}
  Let $\varepsilon$ be the value defined above.  If $m$ satisfies
  both $2\sigma(-m/2) \leq \varepsilon$ and
  $1-\sigma(m/2) \leq \varepsilon$, then $\numpar$ has a solution
  iff $\shap_F(X_0) \geq 1/2(1+\varepsilon)$.
\end{claim}

The claim implies Theorem~\ref{th:nbn:hard}.  To prove the claim, we express the
\shap\ explanation using Eq.~\eqref{eq:shap:alt}. Let $\mathbf{X}_S$ denote the event
$\bigwedge_{i \in S} (X_i = 1)$. Then, we can write the \shap\ explanation as:
\begin{align*}
  \shap_F(X_0) = & \sum_{S \subseteq [n]}a_{|S|} \left(\E[F \mid \mathbf{X}_{S \cup \set{0}}]-
                   \E[F \mid \mathbf{X}_S]\right)
\end{align*}

Obviously, $\E[X_0 \mid \mathbf{X}_{S \cup \set{0}}]=1$. In addition, we have
$\sum_{S \subseteq [n]} a_{|S|} = 1$, because there are ${n \choose k}$ sets
of size $k$, hence
$\sum_{S \subseteq [n]} a_{|S|} = \sum_{k=0,n} {n \choose k} \cdot
\frac{k!(n-k)!}{(n+1)!} = 1$.  Therefore $\shap_F(X_0) = 1 - D$, where:
\begin{align}
  D & \defeq  \sum_{S \subseteq [n]}a_{|S|} \cdot \E[X_0 \mid \mathbf{X}_S]\label{eq:d}
\end{align}
To compute $D$, we expand:
\begin{align*}
  \E[& X_0 | \mathbf{X}_S] = \pr(X_0|\mathbf{X}_S) =  \frac{\pr(X_0, \mathbf{X}_S)}{\pr(\mathbf{X}_S)}\\
  &= \frac{\prod_{i\in S} \pr(X_i|X_0)\pr(X_0)}{\prod_{i\in S} \pr(X_i|X_0)\pr(X_0)+\prod_{i\in S} \pr(X_i|\bar X_0)\pr(\bar X_0)}\\
  &= \frac{1}{1+\pr(\bar X_0)/\pr(X_0) \cdot \prod_{i \in S} (\pr(X_i|\bar X_0)/\pr(X_i|X_0))}\\
  &= \sigma(\text{weight}(S))
\end{align*}
where:
\begin{align*}
  \sigma(x) \defeq & \frac{1}{1+e^{-x}}&
                                         \text{weight}(S)  \defeq & -\frac{m}{2} - mc + m(\sum_{i\in S} k_i)
\end{align*}
We compute $D$ in Eq.~\eqref{eq:d} by grouping each set $S$ with its
complement $\bar S \defeq [n]-S$:
\begin{align}
  D = & \frac{1}{2}\sum_{S \subseteq [n]} a_{|S|} \left(\sigma(\text{weight}(S))+\sigma(\text{weight}(\bar S))\right)
        \label{eq:sum:d}
\end{align}
If $S$ is a solution to the number partitioning problem, then:
\begin{align*}
  \sigma(\text{weight}(S))+\sigma(\text{weight}(\bar S)) = 2\sigma(-m/2)
\end{align*}
Otherwise, one of $\text{weight}(S)$, $\text{weight}(\bar S)$ is
$\geq m/2$ and the other is $\leq -3m/2$ and therefore:
\begin{align*}
  \sigma(m/2) \leq \sigma(\text{weight}(S))+\sigma(\text{weight}(\bar S)) \leq 1+\sigma(-3m/2)
\end{align*}
As in Sec.~\ref{app:proof:logistic:regression:hard}, we obtain:
\begin{align*}
  S \in P: && 0 \leq\sigma(\text{weight}(S))+\sigma(\text{weight}(\bar S)) & \leq \varepsilon \\
  S \not\in P: &&  1-\varepsilon \leq\sigma(\text{weight}(S))+\sigma(\text{weight}(\bar S))&\leq 1+\varepsilon
\end{align*}

Therefore, using the fact that $\sum_{S \subseteq[n]} a_{|S|} = 1$,
we derive these bounds for the expression~\eqref{eq:sum:d} for $D$:
\begin{itemize}
\item If the number partitioning problem has no solution, then
  $D \geq 1/2(1-\varepsilon)$, and
  $\shap_F(X_0) \leq 1/2(1+\varepsilon)$.
\item Otherwise, let $S$ be any solution to the $\numpar$ problem,
  and $k = |S|$, then:
  \begin{align*}
    D \leq & \left(\frac{1}{2}(1+\varepsilon) - a_k(1+\varepsilon)\right) + a_k \varepsilon\\
    \leq & \frac{1}{2} - \left(a_k - \frac{\varepsilon}{2}\right)<\frac{1}{2}-\frac{\varepsilon}{2}
  \end{align*}
  and therefore $\shap_F(X_0) > 1/2(1+\varepsilon)$.
\end{itemize}

\section{Proof of Corollary~\ref{cor:nbn:hard}}
\label{app:leaf-hard}
\begin{proof}
  (Sketch) We use a reduction from the $\numpar$ problem, as in the
  proof of Theorem~\ref{th:nbn:hard}.  We start by constructing the
  NBN with variables $X_0, X_1, \ldots, X_n$ (as for
  Theorem~\ref{th:nbn:hard}), then add two more variables
  $X_{n+1}, X_{n+2}$, and edges
  $X_0 \rightarrow X_{n+1} \rightarrow X_{n+2}$, and define the random
  variables $X_{n+1}, X_{n+2}$ to be identical to $X_0$,
  i.e. $X_0=X_{n+1}=X_{n+2}$.  The prediction function is $F=X_{n+2}$, i.e. it
  returns the feature $X_{n+2}$, and the variables $X_0,X_{n+1}$ are
  latent.  Thus, the new BN is identical to the NBN, and, since both
  models have exactly the same number of non-latent variables, the
  $\shap$-explanation is the same.
\end{proof}

%% file: secF.tex
\section{Proof of Lemma~\ref{lemma:empirical:hard}~(\ref{item:empirical:hard:1})}

\label{sec:proof:empirical:hard:1}

Fix a PP2CNF $\Phi=\bigwedge (U_i \vee V_j)$.  A {\em symmetric}
probability space is defined by two numbers $p, q\in [0,1]$ and
consists of the fully-factorized distribution where
$\pr(U_1)=\pr(U_2)=\cdots = p$ and $\pr(V_1)=\pr(V_2)=\cdots = q$.  A
{\em quasi-symmetric} probability space consists of two sets of
indices $I, J$ and two numbers $p,q$ such that:
\begin{align*}
  \pr(U_i) = &
 \begin{cases}
   p & \mbox{when } i\not\in I\\
   1 & \mbox{when } i \in I
 \end{cases} &
  \pr(V_j) = &
 \begin{cases}
   q & \mbox{when } j\not\in J\\
   1 & \mbox{when } j \in J
 \end{cases} &\end{align*}
In this and the following section we prove
Lemma~\ref{lemma:empirical:hard}: computing the $\shap$-explanation
over an empirical distribution is polynomial time equivalent to
computing the expectation of PP2CNF formulas over a (quasi)-symmetric
distribution.  \citet{DBLP:journals/siamcomp/ProvanB83} proved that
computing the expectation of a PP2CNF over uniform distributions is
\#P-hard in general.  Since uniform distribution are symmetric (namely
$p=q=1/2$) it follows that computing $\shap$-explanations is \#P-hard
in general.

In this section we prove item ~(\ref{item:empirical:hard:1}) of
Lemma~\ref{lemma:empirical:hard}.  Fix a 0/1-matrix $\bm x$ defining
an empirical distribution, and let $F$ be a real-valued prediction
function over these features.  Let $\Phi_{\bm x}$ be the PP2CNF
associated to $\bm x$ (see Definition~\ref{def:phi:matrix}).  Will
assume w.l.o.g. that $\bm x$ has $n+1$ features (columns), denoted
$X_0, X_1, \ldots, X_n$.  The prediction function $F$ is any function
$F : \set{0,1}^{n+1} \rightarrow \R$. We prove:

\begin{prop} \label{prop:empirical:hard:1} One can compute
  $\shap_F(X_0)$ in polynomial time using an oracle for computing
  $\E[\Phi_{\bm x}]$ over quasi-symmetric distributions.
\end{prop}

Denote by $y_i \defeq F(x_{i0}, x_{i1}, \ldots, x_{in})$, $i=1,m$ the
value of $F$ on the $i$'th row of the matrix $\bm x$.  Since the only
possible outcomes of the probability space are the $m$ rows of the
matrix, the quantity $\shap_F(X_0)$ depends only on the vector
$\bm y \defeq (y_1,\ldots, y_m)$.  Furthermore, by the linearity of
the \shap\ explanation (Eq.~\eqref{eq:shap:linear}), it suffices to
compute the \shap\ explanation in the case when $\bm y$ has a single
value $=1$ and all others are $=0$.  By permuting the rows of the
matrix, we will assume w.l.o.g. that $y_1=1$, and
$y_2=y_3=\cdots=y_m=0$.  In summary, denoting $F_1$ the function that is 1
on the first row of the matrix $\bm x$ and is 0 on all other rows, our
task is to compute $\shap_{F_1}(X_0)$.

For that we use the following expression for $\shap$ (see also
Sec.~\ref{sec:positive-result}):
\begin{align}
&  \shap_{F_1}(X_0) =  \sum_{k=0,n}\frac{k!(n-k)!}{(n+1)!}\Big(\nonumber\\
& \sum_{S \subseteq [n]: |S|=k}\left(\E[F_1|\mathbf{X}_{S \cup \set{0}}=1]-\E[F_1|\mathbf{X}_S=1]\right)\Big)\label{eq:as:discussed}
\end{align}
We will only show how to compute the quantity
\begin{align}
v_{F_1,k} = & \sum_{S \subseteq [n]: |S|=k} \E[F_1|\mathbf{X}_S=1] \label{eq:v:k:repeated}
\end{align}
using an oracle to $\E[\Phi_{\bm x}]$, because the quantity
$\sum_{S: |S|=k} \E[F_1|\mathbf{X}_{S \cup \set{0}}=1]$ is computed similarly,
by restricting the matrix $\bm x$ to the rows $i$ where $x_{i0}=1$.
The PP2CNF $\Phi$ associated to this restricted matrix is obtained
from $\Phi_{\bm x}$ as follows.  Let $I \defeq \setof{i}{x_{i0}=1}$ be
the set of rows of the matrix where the feature $X_0$ is 1.  Then, we
need to remove all clauses of the form $(U_i \vee V_j)$ for $i \in I$.
This is equivalent to setting $U_i := 1$ in $\Phi_{\bm x}$.
Therefore, we can compute the expectation of the restricted $\Phi$ by
using our oracle for $\E[\Phi_{\bm x}]$, and running it over the
probability space where we define $\pr(U_i) \defeq 1$ for all
$i \in I$. Hence, it suffices to show only how to compute the
expression~\eqref{eq:v:k:repeated}.  Notice that the quantity
$v_{F_1,k}$ is the same as what we defined earlier in
Eq.~\eqref{eq:v:g:k}.

Column $X_0$ of the matrix is not used in
expression~\eqref{eq:v:k:repeated}, because the set $S$ ranges over
subsets of $[n]$.  Hence w.l.o.g. we can drop feature $X_0$ and denote
by $\bm x$ (with some abuse) the matrix that only has the features
$X_1, \ldots, X_n$.  In other words,
$\bm x \in \set{0,1}^{m \times n}$.  The PP2CNF formula for the
modified matrix is obtained from $\Phi_{\bm x}$ by setting $V_0 := 1$,
hence we can compute its expectation by using our oracle for
$\E[\Phi_{\bm x}]$.

We introduce the following quantities associated to the matrix
$\bm x \in \set{0,1}^{m \times n}$:
\begin{itemize}
\item For all $S \subseteq [n]$, $\ell \leq m, k \leq n$, we
  define:
  \begin{align}
    g(S) \defeq & \setof{i}{\forall j \in S, x_{ij} = 1}\label{eq:g:s}\\
    a_{\ell k} \defeq & |\setof{S}{|S|=k, |g(S)|=\ell}|\label{eq:a:ell:k}
  \end{align}
\item We define the sequence $v_k$, $k=0,1,\ldots,n$:
  \begin{align}
    v_k \defeq & \sum_{l=1,m} \frac{a_{\ell k}}{\ell} \label{eq:v:k}
  \end{align}
\item We define the value $V$:
  \begin{align}
    V \defeq & \sum_{k=0,n}\frac{k!(n-k)!}{(n+1)!}v_k \label{eq:v:v}
  \end{align}
\end{itemize}

We prove that, under a certain condition, the value $v_k$ in
Eq.~\eqref{eq:v:k} is equal to Eq.~\eqref{eq:v:k:repeated}; this
justifies the notation $v_k$, since it turns out to be the same as
$v_{F_1,k}$ from Eq.~\eqref{eq:v:g:k}.  
\begin{defn} \label{def:x:good} Call the matrix $\bm x$ ``good'' if
  $\forall i, j$, $x_{1j} \geq x_{ij}$.
\end{defn}
In other words, the matrix is ``good'' if the first row dominates all
others.  In general the matrix $\bm x$ need not be ``good'', however
we can make it ``good'' by removing all columns where row 1 has a value 0.
More precisely, let $J^{(1)} \defeq \setof{j}{x_{1j}=1}$ denote the
non-zero positions of the first row, and let $\bm x^{(1)}$ denote the
sub-matrix of $\bm x$ consisting of the columns $J^{(1)}$.  Obviously,
$\bm x^{(1)}$ is ``good'', because its first row is $(1,1,\ldots,1)$.
The following hold:
\begin{align*}
\mbox{If } S \subseteq J^{(1)}: && \E_{\bm x}[F_1 | \mathbf{X}_S = 1] =& \E_{\bm  x^{(1)}}[F_1 | \mathbf{X}_S=1]\\
\mbox{If } S \not\subseteq J^{(1)}: && \E_{\bm x}[F_1 | \mathbf{X}_S = 1] =&0
\end{align*}
(When $S\not\subseteq J^{(1)}$ then the quantity
$\E_{\bm x^{(1)}}[F_1 | \mathbf{X}_S=1]$ is undefined).  Therefore:
\begin{align*}
   \sum_{S\subseteq [n]: |S|=k}\!\!\! \E_{\bm x}[F_1|\mathbf{X}_S=1] = & \sum_{S\subseteq J^{(1)}: |S|=k} \!\!\! \E_{\bm x^{(1)}}[F_1|\mathbf{X}_S=1] 
\end{align*}
It follows that, in order to compute the values in
Eq.~\eqref{eq:v:k:repeated}, we can consider the matrix $\bm x^{(1)}$
instead of $\bm x$; its associated PP2CNF is obtained from
$\Phi_{\bm x}$ by setting $V_j := 1$ for all $j \in [m]-J^{(1)}$,
hence we can compute its expectation over a quasi-symmetric space by
using our oracle for computing $\E[\Phi_{\bm x}]$ over quasi-symmetric
spaces.  To simplify the notation, we will still use the name $\bm x$
for the matrix instead of $\bm x^{(1)}$, and assume w.l.o.g.  that the
first row of the matrix $\bm x$ is $(1,1,\ldots,1)$.

We prove that, when $\bm x$ is ``good'', then $v_k$ is indeed the
quantity Eq.~\eqref{eq:v:k:repeated} that we want to compute.  This
holds for any ``good'' matrix, not just matrices with $(1,1,\ldots,1)$
in the first row, and we need this more general result later in
Sec.~\ref{sec:proof:empirical:hard:2}.

\begin{claim} \label{claim:good:x:makes:vk:work}
  If the matrix $\bm x$ is ``good'', then, for any $k=0,n$:
  \begin{align*}
    v_k = & \sum_{S: |S|=k} \E[F_1|\mathbf{X}_S=1]
  \end{align*}
\end{claim}

\begin{proof}
  Recall that $J^{(1)} \defeq \setof{j}{x_{1j}=1}$.  Let
  $S \subseteq [n]$ be any set of columns.  We consider two cases,
  depending on whether $S$ is a subset of $J^{(1)}$ or not:
  \begin{align*}
S \subseteq & J^{(1)}:    &   |g(S)| > & 0 & \E[F_1|\mathbf{X}_S=1] = & \frac{1}{|g(S)|}\\
S\not\subseteq & J^{(1)}: &  |g(S)| = & 0 & \E[F_1|\mathbf{X}_S=1] = & 0
  \end{align*}
  Therefore:
  \begin{align*}
&   \sum_{S\subseteq [n]: |S|=k} \E[F_1|\mathbf{X}_S=1] =   \sum_{S\subseteq J^{(1)}: |S|=k} \E[F_1|\mathbf{X}_S=1]\\
& = \sum_{S\subseteq J^{(1)}: |S|=k} \frac{1}{|g(S)|} 
= \sum_{S: |S|=k, |g(S)|>0} \frac{1}{|g(S)|} 
= \sum_{\ell > 0} \frac{a_{\ell k}}{\ell}
  \end{align*}
\end{proof}

At this point we introduce two polynomials, $P$ and $Q$.

\begin{defn} \label{def:polynomials:p:q}
  Fix an $m \times n$ matrix $\bm x$ with $0,1$-entries.  The
  polynomials $P(u,v)$ and $Q(u,v)$ in real variables $u,v$ associated
  to the matrix $\bm x$ are the following:
\begin{align*}
  P(u,v) \defeq & \sum_{S \subseteq[n]} u^{|g(S)|}v^{|S|}\\
  Q(u,v) \defeq & \sum_{\scriptsize
                  \begin{array}{c}
                    T \subseteq [m],
                    S\subseteq[n]:\\
                    \forall (i,j)\in T \times S:\ x_{ij}=1
                  \end{array}}
  u^{|T|}v^{|S|}
\end{align*}
\end{defn}
The polynomials are defined by summing over exponentially many sets
$S \subseteq [n]$, or pairs of sets
$S \subseteq [n], T \subseteq [m]$.  In the definition of $P$, we use
the function $g(S)$ associated to the matrix $\bm x$, see
Eq.~\eqref{eq:g:s}.  In the definition of $Q(u,v)$ we sum only those
pairs $T,S$ where $\forall i \in T$, $\forall j \in S$, $x_{ij}=1$.
While their definition involves exponentially many terms, these
polynomials have only $(m+1)(n+1)$ terms, because the degrees of the
variables $u,v$ are $m$ and $n$ respectively.  We claim that these
terms are as follows:
\begin{claim} \label{claim:identities:of:p:and:q}
  The following identities hold:
  \begin{align*}
    P(u,v) = & \sum_{\ell=0,m; k=0,n}a_{\ell k}u^\ell v^k\\
    Q(u,v) = & P(1+u,v)
  \end{align*}
\end{claim}

\begin{proof}  The identity for $P(u,v)$ follows immediately from the
  definition of $a_{\ell k}$.  We prove the identity for $Q$.  From
  the definition of $g(S)$ in Eq.~\eqref{eq:g:s} we derive the
  following  equivalence:
  \begin{align*}
& (\forall i \in T, \forall j \in S: x_{ij} = 1) & \Leftrightarrow &&
T \subseteq g(S)
  \end{align*}
Which implies:
  \begin{align*}
      Q(u,v) = & \sum_{S\subseteq[n], T\subseteq g(S)} u^{|T|}v^{|S|}
  \end{align*}
  and the claim follows from
  $\sum_{T \subseteq g(S)} u^{|T|} = (1+u)^{|g(S)|}$.
\end{proof}

Thus, in order to compute the quantities $v_k$ for $k=0,1,\ldots,n$ it
suffices to compute the coefficients $a_{\ell k}$ of the polynomial
$P(u,v)$, and, for that, it suffices to compute the coefficients of
the polynomial $Q(u,v)$.  For that, we establish the following
important connection between $\E[\Phi_{\bm x}]$ and the polynomial
$Q(u,v)$.  Fix $u, v > 0$ any two positive real values, and let
$p\defeq 1/(1+u)$, $q \defeq 1/(1+v)$; notice that $p,q \in (0,1)$.
Consider the probability space over independent Boolean variables
$U_1, \ldots, U_m, V_1, \ldots, V_n$ where $\forall i\in [m]$,
$\pr(U_i)=p$, and $\forall j \in [n]$, $\pr(V_j)=q$.  Then:

\begin{claim} \label{claim:pp2cnf:to:coefficients:of:q} Given the
  notations above, the following identity holds:
  \begin{align}
    \E[\Phi_{\bm x}] = \frac{1}{(1+u)^m(1+v)^n}Q(u,v) \label{eq:pr:phi:q}
  \end{align}
\end{claim}

\begin{proof}
  A truth assignment for $\Phi_{\bm x}$ consists of two assignments,
  $\theta \in \set{0,1}^m$ for the variables $U_i$, and
  $\tau \in \set{0,1}^n$ for the variables $V_j$.  Defining
  $T \defeq \setof{i}{\theta(U_i)=0}$ and
  $S \defeq \setof{j}{\tau(V_j)=0}$, we observe that
  $\Phi_{\bm x}[\theta,\tau]=\texttt{true}$ iff
  $\forall i \in T, \forall j \in S$, $x_{ij}=1$, and therefore:
  \begin{align*}
    \pr&(\Phi_{\bm x}) = \sum_{\theta,\tau: \Phi[\theta,\tau]=1}\pr(\theta)\pr(\tau)\nonumber\\
    &= \sum_{\scriptsize
        \begin{array}{l}
          T \subseteq [m],
          S\subseteq[n]\\
          \forall (i,j)\in T \times S: x_{ij}=1
        \end{array}}\hspace{-2em}
    p^{m-|T|}(1-p)^{|T|}q^{n-|S|}(1-q)^{|S|}\nonumber\\
    &=  p^mq^n Q((1-p)/p, (1-q)/q) 
  \end{align*}
\end{proof}

Finally, to prove
Lemma~\ref{lemma:empirical:hard}~(\ref{item:empirical:hard:1}), it
suffices to show how to use an oracle for $E[\Phi_{\bm x}]$ to compute
the coefficients of the polynomial $Q(u,v)$.  We denote by
$b_{\ell k}$ these coefficients, in other words:
\begin{align}
  Q(u,v) = & \sum_{\ell=0,m; k=0,n} b_{\ell k} u^\ell v^k \label{eq:polynomial:q:only}
\end{align}
To compute the coefficients $b_{\ell k}$, we proceed as follows.
Choose $m+1$ distinct values $u_0, u_1, \ldots, u_m > 0$, and choose
$n+1$ distinct values $v_0, v_1, \ldots, v_n > 0$, and for all
$i = 0,m$ and $j=0,n$, use the oracle for $\E[\Phi_{\bm x}]$ to
compute $Q(u_i,v_j)$ as per identity~\eqref{eq:pr:phi:q}.  This leads
to a system of $(m+1)(n+1)$ equations whose unknowns are the
coefficients $b_{\ell k}$ (see Eq.~\eqref{eq:polynomial:q:only}) and
whose coefficients are $u_i^\ell v_j^k$.  The matrix $\bm A$ of this
system of equations is an $[(m+1)(n+1)] \times [(m+1)(n+1)]$ matrix,
whose rows are indexed by pairs $(i,j)$, and whose columns are indexed
by pairs $(\ell,k)$:
\begin{align*}
  A_{(ij),(\ell k)} = & u_i^\ell v_j^k
\end{align*}
We prove that this matrix is non-singular, and for that we observe
that it is the Kronecker product of two Vandermonde matrices.  Recall
that the $t \times t$ Vandermonde matrix defined by $t$ numbers $z_1,
\ldots, z_t$ is:
\begin{align*}
  V(z_1, \ldots, z_t) = &
\left[
 \begin{array}{cccc}
   1 & 1 & \ldots & 1 \\
   z_1 & z_2 & \ldots & z_t \\
   z_1^2 & z_2^2 & \ldots & z_t^2 \\
     & & \ldots & \\
   z_1^{t-1} & z_2^{t-1} & \ldots & z_t^{t-1}
 \end{array}
\right]
\end{align*}
It is known that
$\det(V(z_1, \ldots, z_t))=\prod_{1 \leq i < j \leq t}(z_j - z_i)$ and
this is $\neq 0$ iff the values $z_1, \ldots, z_t$ are distinct.  We
observe that the matrix $\bm A$ is the Kronecker product of two
Vandermonde matrices:
\begin{align*}
  \bm A = & V(u_0,u_1,\ldots,u_m) \otimes V(v_0,v_1,\ldots,v_n)
\end{align*}
Since we have chosen $u_0, \ldots, u_m$ to be distinct, and similarly
for $v_0, \ldots, v_n$, it follows that both Vandermonde matrices are
non-singular, hence $\det(\bm A) \neq 0$.  Thus, we can solve this
linear system of equations in time $O\left(\left((m+1)(n+1)\right)^3\right)$, and
compute all coefficients $b_{\ell k}$.

\medskip

{\bf Putting It Together} We prove now
Proposition~\ref{prop:empirical:hard:1}.  We are given a 0/1 matrix
$\bm x$ with $n+1$ features $X_0, \ldots, X_n$ and $m$ rows.  To
compute $\shap_F(X_0)$ we proceed as follows:
\begin{enumerate}
\item \label{item:shap:to:pp2cnf:1} For each $i=1,m$, compute
  $\shap_{F_i}(X_0)$, where $F_i$ is the function defined as $=1$ on
  row $i$ of the matrix, and $=0$ on all other rows of the matrix.
  Return $\shap_F(X_0) = \sum_{i=1,m} y_i \shap_{F_i}(X_0)$, where
  $y_i \defeq F(x_{i0}, x_{i1}, \ldots, x_{in})$ is the value of $F$
  on the $i$'th row of the matrix.
\item \label{item:shap:to:pp2cnf:2} To compute $\shap_{F_i}(X_0)$, switch rows $1$ and $i$ of the
  matrix, and compute $\shap_{F_1}(X_0)$ on the modified matrix.
\item \label{item:shap:to:pp2cnf:3} To compute $\shap_{F_1}(X_0)$,
  compute both sums in Eq.~\eqref{eq:as:discussed}.
\item \label{item:shap:to:pp2cnf:4} To compute
  $\sum_{S \subseteq [n]: |S|=k}\E[F_1|\mathbf{X}_S=1]$, perform steps
  (\ref{item:shap:to:pp2cnf:5}) to (\ref{item:shap:to:pp2cnf:8})
  below.
\item \label{item:shap:to:pp2cnf:5} Let
  $J^{(1)} = \setof{j}{j \in [n], x_{1j}=1}$; notice that
  $0 \not\in J^{(1)}$.  Let $n^{(1)} = |J^{(1)}|$.  Let $\Phi'$ denote
  the PP2CNF obtained from $\Phi_{\bm x}$ by setting $V_j:=1$ for all
  $j \not\in J^{(1)}$.  Thus, $\Phi'$ has $m+n^{(1)}$ variables: $U_i$
  for $i \in [m]$, and $V_j$ for $j \in J^{(1)}$.
\item \label{item:shap:to:pp2cnf:6} Choose distinct values
  $u_0, u_1, \ldots, u_m \in (0,1)$ and distinct values
  $v_0, v_1, \ldots, v_{n^{(1)}} \in (0,1)$.  For each fixed
  combination $u_\alpha, v_\beta$, compute
  $Q(u_\alpha, v_\beta) =
  (1+u_\alpha)^m(1+v_\beta)^{n^{(1)}}\E[\Phi']$ (see
  Claim~\ref{claim:pp2cnf:to:coefficients:of:q}).  The value
  $\E[\Phi']$ over the probability space where, for all $i,j$:
  $\pr(U_i)=u_\alpha$, $\pr(V_j)=v_\beta$: this can be done by
  computing $\E[\Phi_{\bm x}]$ over a quasi-symmetric space.
\item \label{item:shap:to:pp2cnf:7} Using the $(m+1)(n^{(1)}+1)$
  results from the previous step, form a system of Equations where the
  unknowns are the coefficients $b_{\ell k}$, $\ell=0,m$,
  $k=0,n^{(1)}$, of the polynomial $Q(u,v)$,
  see~\eqref{eq:polynomial:q:only}.  Solve for the coefficients
  $b_{\ell k}$.
\item \label{item:shap:to:pp2cnf:8} Compute the coefficients
  $a_{\ell k}$ of the polynomial $P(u,v)=Q(u-1,v)$, see
  Claim~\ref{claim:identities:of:p:and:q}, then compute
  $v_k = \sum_\ell a_{\ell k}/\ell$.  By
  Claim~\ref{claim:good:x:makes:vk:work},
  $v_k = \sum_{S: |S|=k} \E[F_1|\mathbf{X}_S=1]$, completing Step
  (\ref{item:shap:to:pp2cnf:4}).
\item \label{item:shap:to:pp2cnf:9} To compute
  $\sum_{S \subseteq [n]: |S|=k}\E[F_1|\mathbf{X}_{S \cup \set{0}}=1]$, first
  set $U_i:=0$ for all rows $i$ where $x_{i0}=0$, then repeat steps
  (\ref{item:shap:to:pp2cnf:5}) to (\ref{item:shap:to:pp2cnf:8}).
\item This completes Step (\ref{item:shap:to:pp2cnf:3}), and we obtain
  $\shap_{F_1}(X_0)$.
\end{enumerate}

%% file: secG.tex
\section{Proof of Lemma~\ref{lemma:empirical:hard}~(\ref{item:empirical:hard:2})}

\label{sec:proof:empirical:hard:2}

Here we prove item (\ref{item:empirical:hard:2}) of
Lemma~\ref{lemma:empirical:hard}: one can compute $\E[\Phi]$ over a
quasi-symmetric probability space in polynomial time, given an oracle
for $\shap$ on empirical distributions.  If the probability space sets
$\pr(U_i)=1$ for some variable, then we can simply replace $\Phi$ with
$\Phi[U_i:=1]$, and similarly if $\pr(V_j)=1$.  Hence, w.l.o.g., we
can assume that the probability space is symmetric.

More precisely, we fix a PP2CNF formula
$\Phi = \bigwedge (U_i \vee V_j)$, and let
$p=\pr(U_1)=\cdots = \pr(U_m)$ and $q =\pr(V_1) = \cdots = \pr(V_n)$
define a symmetric probability space.  Our task is to compute
$\E[\Phi]$ over this space, given an oracle for computing
$\shap$-explanations on empirical distributions.  Throughout this
section we will use the notations introduced in
Sec.~\ref{sec:proof:empirical:hard:1}.

Let $\bm x$ the matrix associated to $\Phi$: $x_{ij}=0$ iff $\Phi$
contains a clause $U_i \vee V_j$.  We describe our algorithm for
computing $\E(\Phi)$ in three steps.

{\bf Step 1:} $\E[\Phi] \leq^P (v_0,v_1,\ldots,v_k)$.  More
precisely:, we claim that we can compute $\pr(\Phi)$ using an oracle
for computing the quantities $v_0, v_1, \ldots, v_n$ defined in
Eq.~\eqref{eq:v:k}.  We have seen in Eq.~\eqref{eq:pr:phi:q} that
$\E[\Phi] = \frac{1}{(1+u)^m(1+v)^n}Q(u,v)$ where $u=(1-p)/p$ and
$v =(1-q)/q$.  From Claim~\ref{claim:identities:of:p:and:q} we know
that $Q(u,v)=P(1+u,v)$, and the coefficients of $P(u,v)$ are the
quantities $a_{\ell k}$ defined in Eq.~\eqref{eq:a:ell:k}. To complete
Step 1, we will describe a polynomial time algorithm that computes the
quantities $a_{\ell k}$ associated to our matrix $\bm x$, with access
to an oracle for computing the quantities $v_0, \ldots, v_k$
associated to any matrix $\bm x'$.

Starting from the matrix $\bm x$, construct $m+1$ new matrices,
denoted by $\bm x^{(1)}, \bm x^{(2)}, \ldots, \bm x^{(m+1)}$, where,
for each $\Gamma=1,m+1$, $\bm x^{(\Gamma)}$ consists of the matrix
$\bm x$ extended with $\Gamma$ rows consisting of $(1,1,\ldots,1)$.
That is, the matrix $\bm x^{(\Gamma)}$ has $\Gamma+m$ rows, the first
$\Gamma$ rows are $(1,\ldots,1)$, and the remaining $m$ rows are those
in $\bm x$.  We run our oracle to compute the quantities $v_k$ on each
matrix $\bm x^{(\Gamma)}$.  We continue to use the notations
$g(S), a_{\ell k}, v_k$ introduced in Equations~\eqref{eq:g:s},
\eqref{eq:a:ell:k}, \eqref{eq:v:k} for the matrix $\bm x$, and add the
superscript $(\Gamma)$ for the same quantities associated to the
matrix $\bm x^{(\Gamma)}$.  We observe:
\begin{align*}
  g^{(\Gamma)} = & g(S) \cup \set{\mbox{the $\Gamma$ new rows}}\\
  a_{\ell +\Gamma,k}^{(\Gamma)} = & a_{\ell k}
\end{align*}
and therefore:
\begin{align*}
  v^{(1)}_k & = \frac{1}{1}a_{0k} + \frac{1}{2}a_{1k} + \cdots + \frac{1}{m+1}a_{mk}\\
  v^{(2)}_k & = \frac{1}{2}a_{0k} + \frac{1}{3}a_{1k} + \cdots + \frac{1}{m+2}a_{mk}\\
              & \cdots \\
  v^{(m+1)}_k & = \frac{1}{m+2}a_{0k} + \frac{1}{m+3}a_{1k} + \cdots + \frac{1}{2m+2}a_{mk}
\end{align*}
By solving this system of equations, we compute the quantities
$a_{\ell k}$ for $\ell = 0, m$.  The matrix of this system is a
special case of Cauchy's double alternant determinant:
\begin{align*}
  \det\left[
\frac{1}{x_i+y_j}
\right] = &
\frac{\prod_{1 \leq i < j \leq n}(x_i-x_j)(y_i-y_j)}{\prod_{i,j}(x_i+y_j)}
\end{align*}
where $x_i=i$ and $y_j=j-1$, and therefore the matrix of the system is
non-singular.

We observe that all matrices $\bm x^{(1)}, \ldots, \bm x^{(m+1)}$ are
``good'' (see Definition~\ref{def:x:good}), because their first row is
$(1,\ldots,1)$.

{\bf Step 2:} Let $\bm x$ be a ``good'' matrix
(Definition~\ref{def:x:good}).
Then: $(v_0,v_1,\ldots,v_n) \leq^P V$ ($V$ defined in Eq.~\eqref{eq:v:v}).
In other words, given a matrix $\bm x$, we claim that we can compute
the quantities $v_0, v_1, \ldots, v_n$ associated to $\bm x$ by
Eq.~\eqref{eq:v:k} in polynomial time, given access to an oracle for
computing the quantity $V$ associated to any matrix $\bm x'$.
The algorithm proceeds as follows. For each
$\Delta=0,1,\ldots,n$, construct a new $m \times (2n)$ matrix
$x^{(\Delta)}$ by extending $\bm x$ with $\Delta$ new columns set to
$1$ and $n-\Delta$ new columns set to $0$. Thus, $x^{(\Delta)}$ is:
\begin{align*}
&  \left(
  \begin{array}{ccccccccccc}
    x_{11}&x_{12}&\ldots &x_{1n}&1&1&\ldots&1&0&\ldots&0\\
    x_{21}&x_{22}&\ldots &x_{2n}&1&1&\ldots&1&0&\ldots&0\\
         &     &\ldots &     & & & & & & & \\
    x_{m1}&x_{m2}&\ldots &x_{mn}&1&1&\ldots&1&0&\ldots&0\\
  \end{array}
\right)
\end{align*}
Notice that $x^{(\Delta)}$ is ``good'', for any $\Delta$.  We run the
oracle on each matrix $x^{(\Delta)}$ to compute the quantity
$V^{(\Delta)}$.  We start by observing the following relationships
between the parameters of the matrix $\bm x$ and those of the matrix
$\bm x^{(\Delta)}$:
\begin{align*}
  g^{(\Delta)}(S) = & g(\Delta \cap [n]) \\
  a^{(\Delta)}_{\ell p} = & \sum_{k=0,\min(p,n)}{{\Delta \choose p-k}} a_{\ell k}\\
  v^{(\Delta)}_p = & \sum_{k=0,\min(p,n)} {{\Delta \choose p-k}} v_k
\end{align*}
Notice that, when $p > n+\Delta$, then $v_p^{(\Delta)}=0$.  We use the
oracle to compute the quantity $V^{(\Delta)}$, which is:
\begin{align*}
  V^{(\Delta)} = & \sum_{p=0,2n}\frac{p!(2n-p)!}{(2n+1)!}v_p^{(\Delta)} \\
 = &  \frac{1}{2n+1} \sum_{p=0,n+\Delta}\frac{1}{{2n  \choose p}}v^{\Delta}_p\\
 = & \frac{1}{2n+1} \sum_{p=0,n+\Delta}\sum_{k=0,\min(p,n)}\frac{{\Delta \choose p-k}}{{2n  \choose p}}v_k\\
 = & \frac{1}{2n+1} \sum_{k=0,n}\sum_{p=k,k+\Delta}\frac{{\Delta \choose p-k}}{{2n  \choose p}}v_k\\
 = & \frac{1}{2n+1} \sum_{k=0,n}\left(\sum_{q=0,\Delta}\frac{{\Delta \choose q}}{{2n  \choose k+q}}\right)v_k\\
  \defeq & \frac{1}{2n+1} \sum_{k=0,n}A_{\Delta,k}\cdot v_k
\end{align*}
Thus, after running the oracle on all matrices
$\bm x^{(0)}, \ldots, \bm x^{(n)}$, we obtain a system of $n+1$
equations with $n+1$ unknowns $v_0, v_1, \ldots, v_n$.  It remains to
prove that system's matrix, $A_{\Delta, k}$, is non-singular matrix.
Let us denote following matrices by:
\begin{align*}
A_{\Delta, k} \defeq & \sum_{q=0,\Delta}  \frac{{\Delta \choose q}}{{2n  \choose k+q}} & \Delta=0,n; k=0,n;\\
B_{\Delta, q} \defeq & {\Delta \choose q} & \Delta=0,n; q=0,n;\\
C_{q,k} \defeq & \frac{1}{{2n \choose k+q}} & q=0,n; k=0,n;
\end{align*}
It is immediate to verify that $\bm A = \bm B \cdot \bm C$, so it
suffices to prove $\det(\bm B)\neq 0$, $\det(\bm C)\neq 0$.  We start
with $\bm B$, and for that consider the Vandermonde matrix
$\bm X \defeq V(x_0, x_1, \ldots, x_n)$, $X_{qt} \defeq x_t^q$.
Denoting $\bm Y \defeq \bm B\cdot \bm X$, we have that
\begin{align*}
  Y_{\Delta t} = & \sum_{q=0,n} B_{\Delta,q} X_{q,t} = \sum_{q=0,n}{\Delta \choose q}x_t^q= (1+x_t)^\Delta
\end{align*}
is also a Vandermonde matrix $\bm Y = V(1+x_0, 1+x_1, \ldots, 1+x_n)$.
We have $\det(\bm Y) \neq 0$ when $x_0, x_1, \ldots, x_n$ are
distinct, proving that $\det(\bm B) \neq 0$.

Finally, we prove $\det(\bm C)\neq 0$. For that, we prove a slightly
more general result.  For any $N \geq 2n$, denote by $\bm C^{(n,N)}$ the
following $(n+1)\times (n+1)$ matrix:

\begin{align*}
  \bm C^{(n,N)} \defeq &
                         \left(
               \begin{array}{cccc}
                 \frac{1}{{N \choose 0}} &\frac{1}{{N \choose 1}} &\ldots&\frac{1}{{N \choose n}} \\
                 \frac{1}{{N \choose 1}} &\frac{1}{{N \choose 2}} &\ldots&\frac{1}{{N \choose n+1}} \\
                                         &    & \ldots & \\
                 \frac{1}{{N \choose n}} &\frac{1}{{N \choose n+1}} &\ldots&\frac{1}{{N \choose 2n}}
               \end{array}
\right)
\end{align*}
We will prove that $\det(\bm C^{(n,N)}) \neq 0$; our claim follows
from the special case $N=2n$.  For the base case, $n=0$,
$\det(\bm C^{(0,N)})=1$ because $\bm C^{(0,N)}$ is a $1\times 1$
matrix equal to $1/{N \choose 0}$, hence $\det(\bm C^{(0,N)})=1$.  To
show the induction step, we will perform elementary column operations
(which preserve the determinant) to make the last row of the resulting
matrix consist of zeros, except for the last entry.

Consider an arbitrary row $i$, and two adjacent columns $j, j+1$ in
that row:
\begin{align*}
  \begin{array}{cccc}
    \ldots & \frac{1}{{N \choose i+j}} & \frac{1}{{N \choose i+j+1}} & \ldots
  \end{array}
\end{align*}
We use the fact that ${N \choose i+j} = {N \choose i+j+1} \frac{i+j+1}{N-i-j}$
and rewrite the two adjacent elements as:
\begin{align*}
  \begin{array}{cccc}
    \ldots & \left(\frac{1}{{N \choose i+j+1}}  \times \frac{N-i-j}{i+j+1}\right) & \frac{1}{{N \choose i+j+1}} & \ldots
  \end{array}
\end{align*}
Now, for each $j = 0,1,2,...,n-1$, we subtract column $j+1$, multiplied by $\frac{N-(n+j)}{(n+j)+1}$, from column $j$. The last row becomes $0,0,\ldots,0,\frac{1}{{N \choose 2n}}$, which means that $det(C^{(n,N)})$ is equal to $\frac{1}{{N \choose 2n}}$ times the upper left $(n \times n)$ minor.

Now, we check what happens with element at $(i, j)$. After subtraction, it becomes
\begin{align*}
\frac{1}{{N \choose i+j+1}}  \times \left(\frac{N-(i+j)}{(i+j)+1}-\frac{N-(n+j)}{(n+j)+1}\right)
\end{align*}
This expression can be rewritten as:
\begin{align*}
& \frac{1}{{N \choose (i+j)+1}}  \times  \left(\frac{N-(i+j)}{(i+j)+1}-\frac{N-(n+j)}{(n+j)+1}\right)  \\
& = \frac{(N-i-j-1)!(i+j+1)!}{N!} \frac{(N+1)(n-i)}{(i+j+1)(n+j+1)} \\ 
& = \frac{(N-i-j-1)!(i+j)!}{(N-1)!N} \frac{(N+1)(n-i)}{(n+j+1)} \\ 
&= \frac{1}{{N-1 \choose (i+j)}} \frac{(N+1)(n-i)}{N(n+j+1)} 
\end{align*}

Note that this expression holds with the whole $(n \times n)$
upper-left minor of $\bm C^{(n,N)}$: the element in the lower-right
corner of the matrix remains $1/{N \choose 2n}$.  Observe that the
$(i,j)$-th entry of this minor is precisely the $(i,j)$-entry of
$C^{(n-1, N-1)}$, multiplied by $\frac{(N+1)(n-i)}{N(n+j+1)} $.  Here
$\frac{N+1}{N}$ is a global constant, $n-i$ is the same constant in
the entire row $i$, and $\frac{1}{n+j+1}$ is the same constant in the
entire column $j$. We factor out the global constant $\frac{N+1}{N}$,
factor out $n-i$ from each row $i$, and factor out $\frac{1}{n+j+1}$
from each column $j$, and obtain the following recurrence:

\begin{align*}
  \det(\bm C^{(n,N)}) = & \frac{1}{{N \choose 2n}} \left(\frac{N+1}{N}\right)^n \times \frac{\prod_{i=0}^{n-1} (n-i)}{\prod_{j=0}^{n-1}(n+j+1)}\\
& \times\det(\bm C^{(n-1,N-1)})
\end{align*}
It follows by induction on $n$ that $\det(\bm C^{(n,N)})\neq 0$.

{\bf Step 3:} Let $\bm x$ be a ``good'' matrix
(Definition~\ref{def:x:good}).  Then $V \leq^P \shap$. More precisely,
we claim that we can compute the quantity $V$ associated to a matrix
$\bm x$ as defined in Eq.~\eqref{eq:v:v} in polynomial time, by using
an oracle for computing  $\shap_{F_1}(X_j)$ over any matrix $\bm x'$.

We modify the matrix $\bm x$ as follows.  We add a new attribute $X_0$
whose value is 1 only in the first row, and let $F_1 = X_0$ denote the
function that returns the value of feature $X_0$.  We show here the
new matrix $\bm x'$, augmented with the values of the function $F_1$:
  \begin{align*}
&
\left(
  \begin{array}{ccccc|c}
    X_0 & X_1 &X_2 &\ldots&X_n & F_1 \\ \hline
    1 & x_{11} & x_{12} & \ldots & x_{1n} & 1 \\
    0 & x_{21} & x_{22} & \ldots & x_{2n} & 0 \\
      & \ldots & \ldots & \ldots & \ldots & \\
    0 & x_{m1} & x_{m2} & \ldots & x_{mn} & 0 \\
  \end{array}
\right)
\end{align*}
We run our oracle to compute $\shap_{F_1}(X_0)$ over the matrix
$\bm x'$.  The value $\shap_{F_1}(X_0)$ is given by
Eq.~\eqref{eq:as:discussed}, but notice that the matrix $\bm x'$ has
$n+1$ columns, while Eq.~\eqref{eq:as:discussed} is given for a matrix
with $n$ columns.  Therefore, since $\E[F_1|X_{S\cup \set{0}}]=1$ for
any set $S$, we have:
\begin{align*}
  \shap_{F_1}(X_0)= & 1 - \sum_{k=0,n}\frac{k!(n-k)!}{(n+1)!}\E[F_1|\mathbf{X}_S=1]
\end{align*}
Since $\bm x$ is ``good'', so is the new matrix $\bm x'$ and, by
Claim~\ref{claim:good:x:makes:vk:work}, for any $k=0,n$
  \begin{align*}
    \sum_{S: |S|=k} \E[F_1|\mathbf{X}_S=1]= &  v_k
  \end{align*}
This implies that we can use the value $\shap_{F_1}(X_0)$ returned by
the oracle to compute the quantity:
\begin{align*}
  \sum_{k=0,n}\frac{k!(n-k)!}{(n+1)!}\E[F_1|\mathbf{X}_S=1] = \sum_{k=0,n}\frac{k!(n-k)!}{(n+1)!}v_k=V
\end{align*}
which completes Step 3

{\bf Putting It Together} Given a PP2CNF formula $\Phi = \bigwedge
(U_i \vee V_j)$, and two probability values $p = \pr(U_1) = \cdots =
\pr(U_m)$ and $q = \pr(V_1) = \cdots = \pr(V_n)$, to compute
$\E[\Phi]$ we proceed as follows:

\begin{itemize}
\item  Construct the 0,1-matrix associated to $\Phi$, denote it $\bm
  x$.
\item Construct $m+1$ matrices $\bm x^{(\Gamma)}$, $\Gamma = 1,m+1$,
  by adding $\Gamma$ rows $(1,1,\ldots,1)$ at the beginning of the
  matrix.
\item For each matrix $\bm x^{(\Gamma)}$, construct $n+1$ matrices
  $\bm x^{(\Gamma,\Delta)}$, $\Delta = 0,n$, by adding $n$ columns,
  of which the first $\Delta$ columns are 1, the others are 0.
\item For each $\bm x^{(\Gamma,\Delta)}$, construct one new matrix
  $(\bm x^{(\Gamma,\Delta)})'$ by adding a column $(1,0,0,\ldots,0)$.
  Call this new column $X_0$.
\item Use the oracle to compute $\shap_{F_1}(X_0)$.  From here, compute
  the value $V^{(\Gamma,\Delta)}$ associated with the matrix
  $\bm x^{(\Gamma,\Delta)}$.
\item Using the values  $V^{(\Gamma,0)}, V^{(\Gamma,1)}, \ldots,
  V^{(\Gamma,n)}$, compute the values $v_0^{(\Gamma)}, v_1^{(\Gamma)},
  \ldots, v_n^{(\Gamma)}$ associated to the matrix $\bm x^{(\Gamma)}$.
\item For each $k=0,n$, use the values $v_k^{(1)}, v_k^{(2)}, \ldots,
  v_k^{(m+1)}$ to compute the coefficients $a_{0k}, a_{1k}, \ldots,
  a_{mk}$ associated to the matrix $\bm x$.
\item At this point we have all coefficients $a_{\ell k}$ of the
  polynomial $P(u,v)$.
\item Compute the coefficients $b_{\ell k}$ of the polynomial $Q(u,v)
  = P(1+u,v)$.
\item Finally, return
  $\E[\Phi]=\frac{p^mq^n}{(1-p)^m(1-q)^n}Q(\frac{1-p}{p},\frac{1-q}{q})$.
\end{itemize}

This concludes the entire proof of Lemma~\ref{lemma:empirical:hard}.